%% file: main.tex
\documentclass{article}

\usepackage[arXiv]{aistats2020}

\input{math_commands.tex}

\usepackage[utf8]{inputenc} 
\usepackage[T1]{fontenc}    
\usepackage{hyperref}       
\usepackage{url}            
\usepackage{booktabs}       
\usepackage{amsfonts}       
\usepackage{nicefrac}       
\usepackage{microtype}      

\usepackage[square,numbers,compress]{natbib}
\usepackage{color}
\usepackage{amsmath}
\usepackage{wrapfig}
\usepackage{tikz}
\usetikzlibrary{positioning}
\usepackage{algorithm}
\usepackage[noend]{algpseudocode}
\usepackage{siunitx}  
\usepackage{commath}  
\usepackage{mathtools}  
\usepackage{drawmatrix}
\usepackage{amsthm}  
\usepackage[english]{babel}
\usepackage{pifont}
\usepackage{multirow}
\usepackage{placeins}  

\definecolor{mydarkblue}{rgb}{0,0.08,0.45}
\definecolor{mydarkgreen}{rgb}{0.0, .4, 0.0}
\definecolor{mybrown}{rgb}{0.59, 0.29, 0.0}
\hypersetup{
    colorlinks=true,
    linkcolor=mydarkblue,
    citecolor=mydarkblue,
    filecolor=mydarkblue,
    urlcolor=mydarkblue}

\newcommand{\mineq}{\mathrel{-}=}
\newtheorem{thm}{Theorem}  

\newtheorem*{thm*}{Theorem}  
\newtheorem*{lemma*}{Lemma}

\newcommand{\defeq}{\vcentcolon=}

\usepackage{xcolor}

\newcommand{\Real}{\mathbb{R}}  
\newcommand{\ep}{\textbf{w}}  
\newcommand{\hparam}{\boldsymbol{\lambda}}
\newcommand{\hp}{\boldsymbol{\lambda}}
\newcommand{\edim}{m}  
\newcommand{\hdim}{n} 
\newcommand{\edom}{\textbf{W}}  
\newcommand{\xdom}{\mathcal{X}}  
\newcommand{\tdom}{\mathcal{Y}}  
\newcommand{\hdom}{\boldsymbol{\Lambda}} 
\newcommand{\loss}{\mathcal{L}}  
\newcommand{\dataset}{\mathcal{D}}  
\newcommand{\Ltr}{\loss_{\!\text{T}}}  
\newcommand{\Lval}{\loss_{\!\text{V}}}  

\newcommand{\hpderiv}[1]{\pd{#1}{\hp}}
\newcommand{\epderiv}[1]{\pd{#1}{\ep}}
\newcommand{\ehpderiv}[1]{\md{#1}{2}{\ep}{}{\hp}{T}}
\newcommand{\eepderiv}[1]{\md{#1}{2}{\ep}{}{\ep}{T}}

\newcommand{\trainGrad}{\epderiv{\Ltr}}
\newcommand{\trainHess}[1]{\eepderiv{\Ltr}}
\newcommand{\trainMixed}[1]{\ehpderiv{\Ltr}}
\newcommand{\response}{\ep^{\!*}\!(\!\hp\!)}
\newcommand{\responseJacobian}{\hpderiv{\response}}
\newcommand{\LvalResponse}{\Lval\!(\!\hp,\!\response\!)}
\newcommand{\rpm}{\raisebox{.2ex}{$\scriptstyle\pm$}}

\newcommand{\nn}{NN }
\newcommand{\nns}{NNs }
\newcommand{\NN}{NN }
\newcommand{\ho}{HO }
\newcommand{\HO}{HO }
\newcommand{\Ho}{HO }

\begin{document}
    \twocolumn[
        \vspace{-1cm}
        \aistatstitle{Optimizing Millions of Hyperparameters by Implicit Differentiation}
        \vspace{-0.03\textheight}
        \aistatsauthor{Jonathan Lorraine \And Paul Vicol \And David Duvenaud}
        \aistatsaddress{University of Toronto, Vector Institute}
        \vspace{-0.035\textheight}
        \aistatsaddress{\texttt{\{lorraine, pvicol, duvenaud\}@cs.toronto.edu}}
        \vspace{-0.015\textheight}
    ]
    
    \begin{abstract}
        \vspace{-0.4cm}
        We propose an algorithm for inexpensive gradient-based hyperparameter optimization that combines the implicit function theorem (IFT) with efficient inverse Hessian approximations.
        We present results on the relationship between the IFT and differentiating through optimization, motivating our algorithm.
        We use the proposed approach to train modern network architectures with millions of weights and \textit{millions of hyperparameters}.
        We learn a data-augmentation network---where every weight is a hyperparameter tuned for validation performance---that outputs augmented training examples; we learn a distilled dataset where each feature in each datapoint is a hyperparameter; and we tune millions of regularization hyperparameters.
        Jointly tuning weights and hyperparameters with our approach is only a few times more costly in memory and compute than standard training.
    \end{abstract}
    \vspace{-0.4cm}
    \section{Introduction}\label{sec:introduction}
    \vspace{-0.3cm}
        The generalization of neural networks (\nn\!\!s) depends crucially on the choice of hyperparameters.
        Hyperparameter optimization (HO) has a rich history~\citep{schmidhuber1987evolutionary, bengio2000gradient}, and achieved recent success in scaling due to gradient-based optimizers~\citep{domke2012generic, maclaurin2015gradient, franceschi2017forward, franceschi2018bilevel, shaban2018truncated, finn2017model, rajeswaran2019meta, liu2018darts, grefenstette2019meta}.
        There are dozens of regularization techniques to combine in deep learning, and each may have multiple hyperparameters~\cite{kukavcka2017regularization}.
        If we can scale HO to have as many---or more---hyperparameters as parameters, there are various exciting regularization strategies to investigate.
        For example, we could learn a distilled dataset with a hyperparameter for every feature of each input~\cite{maclaurin2015gradient, wang2018dataset}, weights on each loss term~\cite{ren2018learning, kim2018screenernet, zhang2019autoassist}, or augmentation on each input~\cite{cubuk2018autoaugment, xie2019unsupervised}.

        When the hyperparameters are low-dimensional---e.g., \num{1}-\num{5} dimensions---simple methods, like random search, work; however, these break down for medium-dimensional HO---e.g., \num{5}-\num{100} dimensions.
        We may use more scalable algorithms like Bayesian Optimization~\citep{movckus1975bayesian, snoek2012practical, kandasamy2019tuning}, but this often breaks down for high-dimensional HO---e.g., >\num{100} dimensions.
        We can solve high-dimensional \ho problems locally with gradient-based optimizers, but this is difficult because we must differentiate through the optimized weights as a function of the hyperparameters.
        In other words, we must approximate the Jacobian of the best-response function of the parameters to the hyperparameters.
        
        We leverage the Implicit Function Theorem (IFT) to compute the optimized validation loss gradient with respect to the hyperparameters---hereafter denoted the \textit{hypergradient}.
        The IFT requires inverting the training Hessian with respect to the \nn weights, which is infeasible for modern, deep networks.
        Thus, we propose an approximate inverse, motivated by a link to unrolled differentiation~\cite{domke2012generic} that scales to Hessians of large NNs, is more stable than conjugate gradient~\cite{liao2018reviving, shaban2018truncated}, and only requires a constant amount of memory.
        
        Finally, when fitting many parameters, the amount of data can limit generalization.
        There are \textit{ad hoc} rules for partitioning data into training and validation sets---e.g., using \num{10}\% for validation.
        Often, practitioners re-train their models from scratch on the combined training and validation partitions with optimized hyperparameters, which can provide marginal test-time performance increases.
        We verify empirically that standard partitioning and re-training procedures perform well when fitting few hyperparameters, but break down when fitting many.
        When fitting many hyperparameters, we need a large validation partition, which makes re-training our model with optimized hyperparameters vital for strong test performance.
        
        \begin{figure*}[ht!]
            \vspace{-0.02\textheight}
            \centering
            \begin{tikzpicture}
                \centering
                \node (img)[xshift=-1cm]{\includegraphics[trim={1.5cm, .35cm, 1.0cm, 3.1cm},clip, width=.46\linewidth]{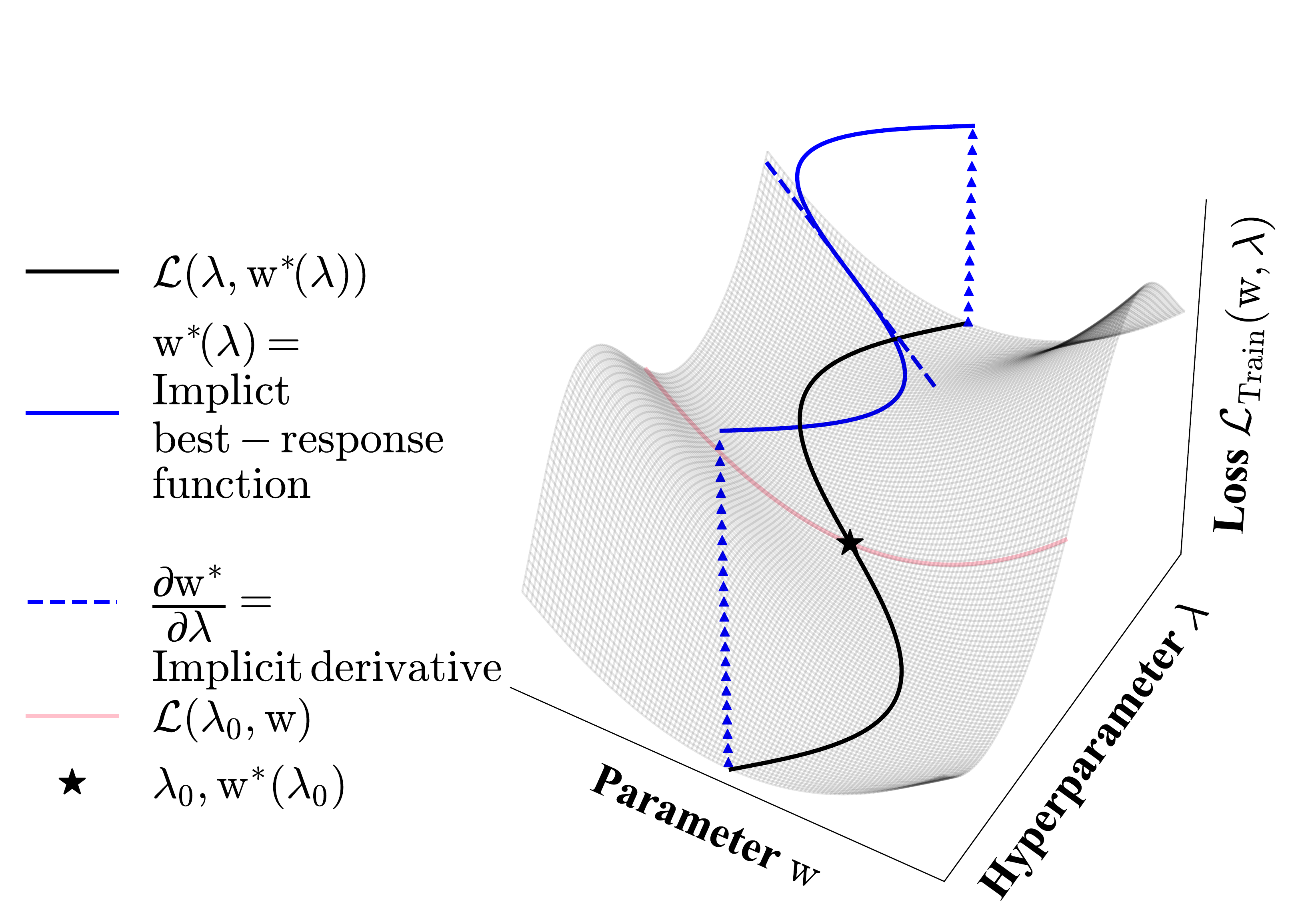}};
                
                \node (img2)[right=of img, xshift=-1.cm]{\includegraphics[trim={3.7cm, .35cm, 1.2cm, 2.6cm},clip, width=.46\linewidth]{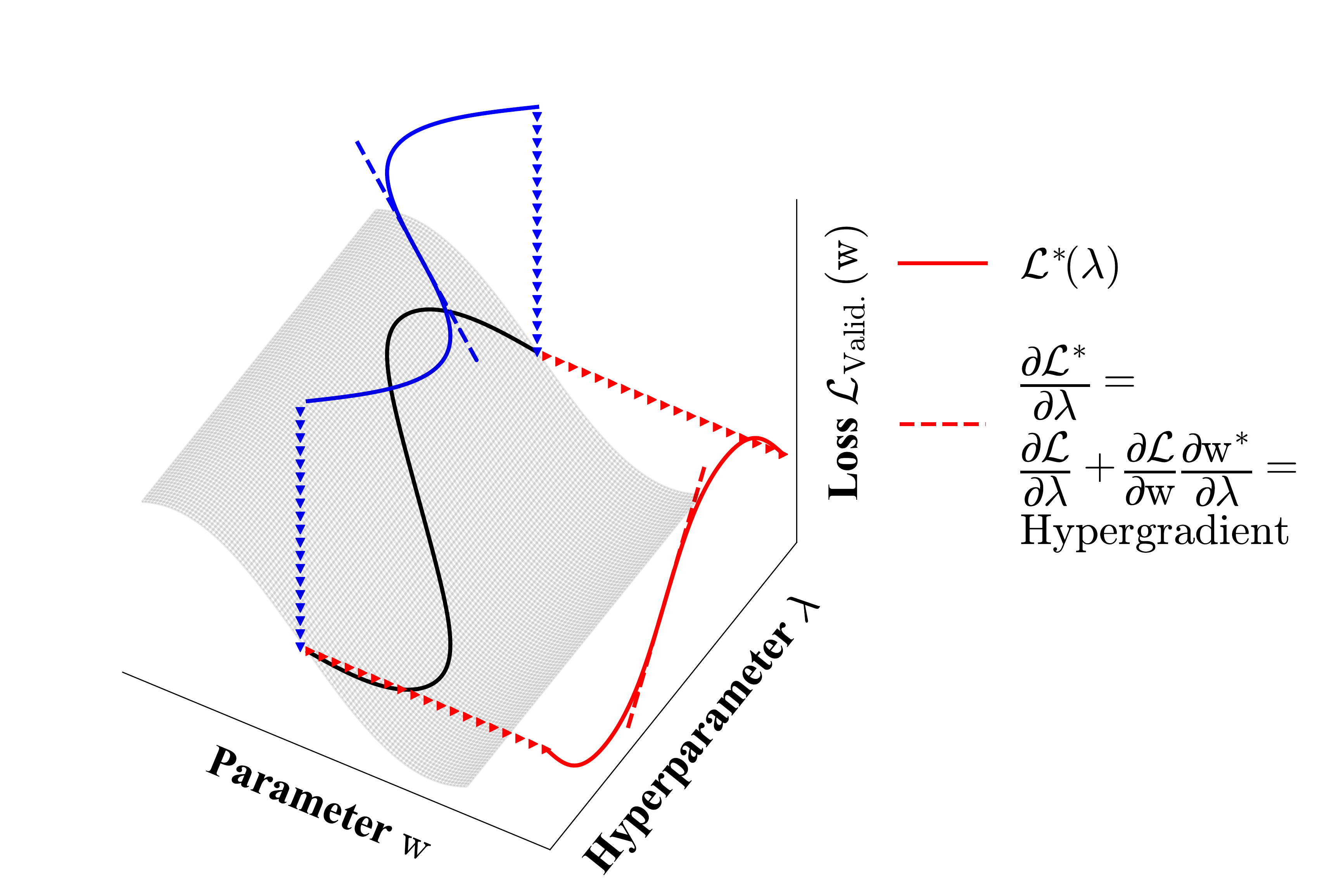}};
            \end{tikzpicture}
            \vspace{-0.015\textheight}
            \caption{
               Overview of gradient-based hyperparameter optimization (HO).
               \emph{Left:} a training loss manifold; \emph{Right:} a validation loss manifold.
               {{The implicit function $\ep^*(\hp)$ is the best-response of the weights to the hyperparameters}} and shown in {\color{blue}blue} projected onto the $(\hp, \ep)$-plane.
               We get our desired objective function {\color{red}$\Lval^{*}(\hp)$} when the best-response is put into the validation loss, shown projected on the hyperparameter axis in {\color{red}red}.
               The validation loss does not depend directly on the hyperparameters, as is typical in hyperparameter optimization.
               Instead, the hyperparameters only affect the validation loss by changing the weights' response.
               We show the best-response Jacobian in {\color{blue}blue}, and the hypergradient in {\color{red} red}.
            }
            \label{fig:train_and_val_manifolds}
            \vspace{-0.015\textheight}
        \end{figure*}
    \vspace{-0.2cm}
    \subsection*{Contributions}
    \vspace{-0.1cm}
        \begin{itemize}
            \item We propose a stable inverse Hessian approximation with constant memory cost.
            \item We show that the IFT is the limit of differentiating through optimization.
            \item We scale IFT-based hyperparameter optimization to modern, large neural architectures, including AlexNet and LSTM-based language models.
            \item We demonstrate several uses for fitting hyperparameters almost as easily as weights, including per-parameter regularization, data distillation, and learned-from-scratch data augmentation methods.
            \item We explore how training-validation splits should change when tuning many hyperparameters.
        \end{itemize}
    \newcommand{\sectionSymbol}{Section}
    
    \section{Overview of Proposed Algorithm}
    \label{sec:background}
        \vspace{-0.3cm}
        There are four essential components to understanding our proposed algorithm.
        Further background is provided in Appendix~\ref{app:background}, and notation is shown in Table~\ref{tab:TableOfNotation}.
        
        \textbf{1. \Ho is nested optimization:}
            Let $\Ltr$ and $\Lval$ denote the training and validation losses, $\ep$ the \nn weights, and $\hp$ the hyperparameters.
            We aim to find optimal hyperparameters $\hp^*$ such that the \nn minimizes the validation loss after training:
            \begin{align} \label{eqn:main-equilibrium}
                &\smash{\hp^*} \!\defeq\! \argmin_{\hp} \smash{{\color{red} \Lval^*\!(\!\hp\!)}} \text{ where } \\
                &{\color{red}\smash{\Lval^*\!(\!\hp\!)}} \!\defeq\! \smash{\LvalResponse \text{ and }}
                \smash{\response} \!\defeq\! \smash{\argmin_{\ep} \Ltr\!(\!\hp,\!\ep\!)}
            \end{align}
            Our implicit function is $\response$, which is the \emph{best-response} of the weights to the hyperparameters.
            We assume unique solutions to $\argmin$ for simplicity.

        \textbf{2. Hypergradients have two terms:}
            For gradient-based \ho we want the hypergradient $\hpderiv{\Lval^*(\hp)}$, which decomposes into:
            \begin{align}\label{eqn:gradient-decomp}
                \begin{split}
                    &\underbrace{\hpderiv{{\color{red} \Lval^*(\hp)}}}_{\text{\tiny{hypergradient}}}
                    = \left. \left( {\color{mydarkgreen}\hpderiv{\Lval}} + \epderiv{\Lval} {\color{blue}\hpderiv{\ep^*}} \right) \right|_{\hp, \ep^*\!(\hp)} =\\ 
                    &{\!\!\!\!\!\! \color{mydarkgreen}\underbrace{\hpderiv{\LvalResponse}}_{\text{\tiny{hyperparam direct grad.}}}} \!\!\!+ \hphantom{A}
                    \overbrace{\!\!\!\!\!\!\!\!
                        \underbrace{\pd{\LvalResponse}{\response}}_{\text{\tiny{parameter direct grad.}}}
                        \times \!\!\!\!
                        {\color{blue}\underbrace{\responseJacobian}_{\text{\tiny{best-response Jacobian}}}}
                    \!\!\!\!\!\!\!\!\!\!\!\!\!\!}^{\text{\tiny{hyperparam indirect grad.}}}
                \end{split}
            \end{align}
        The {\color{mydarkgreen}direct gradient} is easy to compute, but the indirect gradient is difficult to compute because we must account for how the optimal weights change with respect to the hyperparameters (i.e., {\color{blue}$\smash{\responseJacobian}$}).
        In \ho the {\color{mydarkgreen}direct gradient} is often identically \num{0}, necessitating an approximation of the indirect gradient to make any progress (visualized in Fig.~\ref{fig:train_and_val_manifolds}).

        \textbf{3. We can estimate the implicit best-response with the IFT:}
            We approximate the best-response Jacobian---how the optimal weights change with respect to the hyperparameters---using the IFT (Thm. ~\ref{thm:IFT}).
            We present the complete statement in Appendix~\ref{app:IFTComplete}, but highlight the key assumptions and results here.
            \begin{thm}[Cauchy, Implicit Function Theorem]
                If for some \textnormal{$(\hp', \ep'), \left. \smash{\trainGrad} \right|_{\hp', \ep'} = 0$} and regularity conditions are satisfied, then surrounding \textnormal{$(\hp', \ep')$} there is a function \textnormal{$\ep^{*}(\hp)$} s.t. \textnormal{$\left. \smash{\trainGrad} \right|_{\hp, \response} = 0$} and we have:
                \textnormal{
                    \begin{equation}\label{eqn:response-approximation}\tag{IFT}
                        \!\!\left. {\color{blue} \hpderiv{\ep^*}} \!\right|_{\hp'}\!
                        = \!-{\color{magenta} \Big[ \!\!\! \!\!\! \underbrace{\trainHess{\ep}}_{\text{\tiny training Hessian}} \!\!\! \!\!\! \Big]^{-1}}  \times \!\!\!\!\!\!
                        \underbrace{\trainMixed{\ep}}_{\text{training mixed partials}} \!\!\!\!\!\!\!\!\!\!\!\! \Big|_{\hp' \!, \ep^*\!(\hp')}
                    \end{equation}
                }
                \vspace{-0.0225\textheight}
                \label{thm:IFT}
            \end{thm}
            The condition $\smash{\left. \smash{\trainGrad} \right|_{\hp' \! , \! \ep'} = 0}$ is equivalent to $\hp', \ep'$ being a fixed point of the training gradient field.
            Since $\smash{\ep^*(\hp')}$ is a fixed point of the training gradient field, we can leverage the IFT to evaluate the best-response Jacobian locally.
            We only have access to an approximation of the true best-response---denoted $\smash{\widehat{\ep^*}}$---which we can find with gradient descent.

        \textbf{4. Tractable inverse Hessian approximations:}
            To exactly invert a general $m \times m$ Hessian, we often require $\mathcal{O}(m^3)$ operations, which is intractable for the matrix in Eq.~\ref{eqn:response-approximation} in modern \nns\!\!.
            We can efficiently approximate the inverse with the Neumann series:
            \begin{equation}\label{eq:neumann}
                {\color{magenta}\left[ \trainHess{\ep} \right]^{-1}} = {\color{magenta}\lim_{i \to \infty} \sum_{j = 0}^{i} \left[ I - \trainHess{\ep} \right]^j}
            \end{equation}
            In \sectionSymbol~\ref{sec:theory} we show that unrolling differentiation for $i$ steps around locally optimal weights $\ep^*$ is equivalent to approximating the inverse with the first $i$ terms in the Neumann series.
            We then show how to use this approximation without instantiating any matrices by using efficient vector-Jacobian products.
            \begin{figure}
                \vspace{-0.03\textheight}
                \begin{align*}
                    \overbrace{\drawmatrix[fill=red, height=.1, width=.5]\!}^{\hpderiv{{\color{red}\Lval^*}}}
                    =& \,\,\,\,\, \,\,
                    \overbrace{\drawmatrix[height=.1, width=.5, fill=mydarkgreen]\!}^{\color{mydarkgreen}\hpderiv{\Lval}} \,\,\,\,\,\,\,\, + \,\,\,\,\,\,\,
                    \overbrace{\drawmatrix[height=.1]\!}^{\epderiv{\Lval}} \; \,\,\,\,\,\,
                    \overbrace{\drawmatrix[width=.5, fill=blue]\!}^{\color{blue} \smash{\hpderiv{\ep^*}}\vphantom{A_{A_{A_A}}}} 
                    \\
                    =& \,\,\,\,\,\,\,\,
                    \overbrace{\drawmatrix[height=.1, width=.5, fill=mydarkgreen]\!}^{\color{mydarkgreen}\hpderiv{\Lval}} \,\,\,\,\,\,\, +
                    {\color{orange}\underbrace{
                        \overbrace{\drawmatrix[height=.1]\!}^{\color{black}\epderiv{\Lval}} \;
                        \overbrace{\drawmatrix[fill=magenta]\!}^{\color{magenta}\hphantom{a} -\big[\! \smash{\trainHess{\ep}} \!\big]^{\!-\!1}} \;
                    \!\!\!\!\!\!}_{\text{vector-inverse Hessian product}}}\,\,\,\,\,
                    \overbrace{\drawmatrix[width=.5]\!}^{\smash{\trainMixed{\ep}}\vphantom{A_{A_A}}}
                    \\
                    =& \,\,\,\,\,\,\,\,
                    \overbrace{\drawmatrix[height=.1, width=.5, fill=mydarkgreen]\!}^{\color{mydarkgreen}\hpderiv{\Lval}} \,\,\,\,\,\,\, + \,\,\,\,\,\,\,\,\,\,\,\,\,\,\,\,\,\,\,\,\,\,\,\,\,\,
                    \underbrace{\!\!\!\!\!\!
                        {\color{orange} \overbrace{\drawmatrix[height=.1, fill=orange]\!}^{\,\,\,\,\epderiv{\Lval}\times -\big[\! \trainHess{\ep} \!\big]^{\!-\!1}}} \;
                        \overbrace{\drawmatrix[width=.5]\!}^{\smash{\trainMixed{\ep}}\vphantom{A_{A_A}}}
                    \!\!\!}_{\text{vector-Jacobian product}}
                \end{align*}
                \vspace{-0.035\textheight}
                \caption{
                    {Hypergradient computation.  The entire computation can be performed efficiently using vector-Jacobian products, provided a cheap approximation to the inverse-Hessian-vector product is available.}
                }
                \label{fig:equations}
                \vspace{0.2cm}
            \end{figure}
            \begin{algorithm}[h!]
                \caption{Gradient-based \ho for \smash{$\hp^*, \ep^*(\hp^*)$}}
                \label{alg:joint_train}
                \begin{algorithmic}[1]
                    \State Initialize hyperparameters $\hp'$ and weights $\ep'$ 
                    \While{\text{not converged}}
                        \For{$k = 1 \dots N$} 
                            \State $\smash{\ep' \mineq \left. \alpha \cdot \smash{\trainGrad} \right|_{\hp', \ep'}}$ 
                        \EndFor
                        \State $\smash{\hp' \mineq {\color{red}\texttt{hypergradient}(\Lval, \Ltr, \hp', \ep')}}$ 
                    \EndWhile
                    \State \Return $\smash{\hp', \ep'}$ \Comment{$\smash{\hp^*, \ep^*(\hp^*)}$ from \textbf{Eq.\ref{eqn:main-equilibrium}}}
                \end{algorithmic}
            \end{algorithm}
            \vspace{-0.2cm}
            \begin{algorithm}[h!]
                \caption{{\color{red}\smash{\texttt{hypergradient}($\Lval, \Ltr, \hp', \ep'$)}}}
                \label{alg:hyper_gradient}
                \begin{algorithmic}[1]
                    \State $\textbf{v}_1 = \left. \smash{\epderiv{\Lval}} \right|_{\hp', \ep'}$ 
                    \State {\color{orange}$\textbf{v}_2 = \texttt{approxInverseHVP} (\textbf{v}_1, \epderiv{\Ltr})$} 
                    \State $\textbf{v}_3 = \texttt{grad}(\hpderiv{\Ltr}, \ep, \texttt{grad\_outputs}=\textbf{v}_2)$
                    \State \Return $\smash{\left. \smash{\color{mydarkgreen}\hpderiv{\Lval}} \right|_{\hp', \ep'} - \textbf{v}_3}$
                    \Comment{Return to \textbf{Alg.~\ref{alg:joint_train}}}
                \end{algorithmic}
            \end{algorithm}
            \begin{algorithm}[h!]
                \caption{{\color{orange}\texttt{approxInverseHVP}($\textbf{v}, \textbf{f}$)}: Neumann approximation of inverse-Hessian-vector product $\textbf{v}\!\left[\smash{\epderiv{\textbf{f}}}\right]^{-1}$}\label{alg:approxInverseHVP}
                \begin{algorithmic}[1]
                    \State Initialize sum $\textbf{p} = \textbf{v}$
                    \For{$j = 1 \dots i$}
                        \State $\textbf{v} \mineq \alpha \cdot \texttt{grad}(\textbf{f}, \ep, \texttt{grad\_outputs}=\textbf{v})$ 
                        \State $\textbf{p} \mineq \textbf{v}$
                    \EndFor
                    \State \Return $\textbf{p}$  \Comment{Return to \textbf{Alg.~\ref{alg:hyper_gradient}}}.
                \end{algorithmic}
            \end{algorithm}
            
            \vspace{-0.4cm}
            \subsection{Proposed Algorithms}
            \vspace{-0.2cm}
                We outline our method in Algs.~\ref{alg:joint_train},~\ref{alg:hyper_gradient}, and~\ref{alg:approxInverseHVP}, where $\alpha$ denotes the learning rate.
                Alg.~\ref{alg:approxInverseHVP} is also shown in \cite{liao2018reviving}.
                We visualize the hypergradient computation in Fig.~\ref{fig:equations}.
            
    \vspace{-0.1cm}
    \section{Related Work}\label{sec:related-work}
    \vspace{-0.3cm}
        \begin{table*}[h]
            \footnotesize
            \vspace{-0.015\textheight}
            \begin{center}
                \begin{tabular}{c | c | c | r}
                    \textbf{Method} \vphantom{\Big |}
                    & \textbf{Steps}
                    & \textbf{Eval.}
                    & \textbf{Hypergradient Approximation\hphantom{AAAAAA}} \\
                \hline
                    Exact IFT
                    &$\infty$
                    &$\ep^*\!(\hp)$
                    &${\color{mydarkgreen}\hpderiv{\Lval}} - \epderiv{\Lval}\times$\hfill
                        $\left. {\color{magenta}\left[ \trainHess{\ep} \right]^{-1}} \trainMixed{\ep} \right|_{\response}$\,\,\,\\
                    
                    Unrolled Diff.~\citep{maclaurin2015gradient}
                    &$i$
                    &$\ep_0$
                    &${\color{mydarkgreen}\hpderiv{\Lval}} - \epderiv{\Lval}\times$\hfill
                        $\sum_{j\leq i} \! \left[ \! \prod_{k<j} I -  \left. \trainHess{\ep_{i-k}(\hp)} \right|_{\ep_{i-k}} \! \right] \! \left. \trainMixed{\ep_{i-j}(\hp)} \vphantom{A^{A^{A^A}}}\right|_{\ep_{i-j}\hphantom{\,\,}}$\\
                    
                    $L$-Step Truncated Unrolled Diff.~\citep{shaban2018truncated}
                    &$i$
                    &$\ep_L$
                    &${\color{mydarkgreen}\hpderiv{\Lval}} - \epderiv{\Lval}\times$\hfill
                        $\sum_{L \leq j\leq i} \! \left[ \! \prod_{k<j} I -  \left. \trainHess{\ep_{i-k}(\hp)} \right|_{\ep_{i-k}} \! \right] \! \left. \trainMixed{\ep_{i-j}(\hp)} \vphantom{A^{A^{A^A}}} \right|_{\ep_{i-j}\hphantom{\,\,}}$\\
                    
                    \citet{larsen1996design}
                    &$\infty$
                    &$\widehat{\ep^*}\!(\hp)$
                    &${\color{mydarkgreen}\hpderiv{\Lval}} - \epderiv{\Lval}\times$\hfill
                        $\left. {\color{magenta}\left[\! \trainGrad \! \trainGrad^{\! \! T} \!\right]^{-1}} \trainMixed{\ep} \right|_{\widehat{\ep^*}(\hp)}$\\
                    
                    \citet{bengio2000gradient}
                    &$\infty$
                    &$\widehat{\ep^*}\!(\hp)$
                    &${\color{mydarkgreen}\hpderiv{\Lval}} - \epderiv{\Lval}\times$\hfill
                        $\left. {\color{magenta}\left[ \trainHess{\ep} \right]^{-1}} \trainMixed{\ep} \right|_{\widehat{\ep^*}(\hp)}$\\
                    
                    $T1 - T2$~[\citenum{luketina2016scalable}]
                    &$1$
                    &$\widehat{\ep^*}\!(\hp)$
                    &${\color{mydarkgreen}\hpderiv{\Lval}} - \epderiv{\Lval}\times$\hfill
                    $\left.{\color{magenta}\left[I\right]^{-1}} \trainMixed{\ep} \right|_{\widehat{\ep^*}(\hp)}$\\
                    
                    \textbf{Ours}
                    &$i$
                    &$\widehat{\ep^*}\!(\hp)$
                    &${\color{mydarkgreen}\hpderiv{\Lval}} - \epderiv{\Lval}\times$\hfill
                        $\left. {\color{magenta} \left( \sum_{j < i} \left[I - \trainHess{\ep(\hp)}\right]^j \right)} \trainMixed{\ep(\hp)} \right|_{\widehat{\ep^*}(\hp)}$\\
                    
                    Conjugate Gradient (CG) $\approx$
                    &-
                    &$\widehat{\ep^*}\!(\hp)$
                    &${\color{mydarkgreen}\hpderiv{\Lval}} \,\,-$\hfill
                        $\left. {\color{orange}\left(\argmin_{\textbf{x}}\| \textbf{x} \trainHess{\ep} - \epderiv{\Lval} \|\right)} \trainMixed{\ep(\hp)} \right|_{\widehat{\ep^*}(\hp)}$\\
                    
                    Hypernetwork~\citep{lorraine2018stochastic, mackay2019self}
                    &-
                    &-
                    &${\color{mydarkgreen}\hpderiv{\Lval}} + \epderiv{\Lval}\times$\hfill
                        ${\color{blue}\hpderiv{\ep_{\phi}^*}}$ where $\ep_{\phi}^*(\hp) = \argmin_{\phi} \Ltr(\hp, \ep_{\phi}(\hp))$\\
                    
                    Bayesian Optimization~\citep{movckus1975bayesian, snoek2012practical, shaban2018truncated} $\approx$
                    &-
                    &-
                    &$\hpderiv{\mathbb{E}[\color{red}\Lval^{*}]}$ where ${\color{red}\Lval^{*}} \sim \textnormal{Gaussian-Process}(\{\hp_i, \Lval(\hp_i, \ep^{*}(\hp_i))\})$\\
                \end{tabular}
            \end{center}
            \vspace{-0.01\textheight}
            \caption{
                {An overview of methods to approximate hypergradients.}
                Some methods can be viewed as using an approximate inverse in the IFT, or as differentiating through optimization around an evaluation point.
                Here, $\smash{\widehat{\ep^*}(\hp)}$ is an approximation of the best-response at a fixed $\hp$, which is often found with gradient descent.
            }
            \label{tab:response_approxes}
        \end{table*}
        \begin{table}[ht]
            \vspace{-0.015\textheight}
            \begin{center}
                \begin{tabular}{c|c}
                    \textbf{Method} & \textbf{Memory Cost} \\
                    \hline
                    Diff. through Opt.~\cite{domke2012generic, maclaurin2015gradient, shaban2018truncated}
                        & $\mathcal{O}(PI + H)$ \\
                    Linear Hypernet~\cite{lorraine2018stochastic}
                        & $\mathcal{O}(PH)$ \\
                    Self-Tuning Nets (STN)~\cite{mackay2019self}
                        & $\mathcal{O}((P+H)K)$ \\
                    Neumann/CG IFT
                        & $\mathcal{O}(P + H)$
                \end{tabular}
            \end{center}
            \vspace{-0.01\textheight}
            \caption{
                {Gradient-based methods for \ho\!\!.}
                Differentiation through optimization scales with the number of unrolled iterations $I$; the STN scales with bottleneck size $K$, while our method only scales with the weight and hyperparameter sizes $P$ and $H$.
            }
            \label{tab:hyper_methods}
            \vspace{-0.015\textheight}
        \end{table}
        \textbf{Implicit Function Theorem.}
            The IFT has been used for optimization in nested optimization problems~\citep{ochs2015bilevel, anonymous2019ridge}, backpropagating through arbitrarily long RNNs~\citep{liao2018reviving}, or even efficient $k$-fold cross-validation~\citep{beirami2017optimal}.
            Early work applied the IFT to regularization by explicitly computing the Hessian (or Gauss-Newton) inverse~\citep{larsen1996design, bengio2000gradient}.
            In \cite{luketina2016scalable}, the identity matrix is used to approximate the inverse Hessian in the IFT.
            HOAG~\citep{pedregosa2016hyperparameter} uses conjugate gradient (CG) to invert the Hessian approximately and provides convergence results given tolerances on the optimal parameter and inverse.
            In iMAML~\citep{rajeswaran2019meta}, a center to the weights is fit to perform well on multiple tasks---contrasted with our use of validation loss.
            In DEQ~\citep{bai2019deep}, implicit differentiation is used to add differentiable fixed-point methods into \nn architectures.
            We use a Neumann approximation for the inverse-Hessian, instead of CG~\citep{pedregosa2016hyperparameter,rajeswaran2019meta} or the identity.
        
        \textbf{Approximate inversion algorithms.}
            CG is difficult to scale to modern, deep \nns\!\!.
            We use the Neumann inverse approximation, which was observed to be a stable alternative to CG in \nns~\citep{liao2018reviving, shaban2018truncated}.
            The stability is motivated by connections between the Neumann approximation and unrolled differentiation~\citep{shaban2018truncated}.
            Alternatively, we could use prior knowledge about the \nn structure to aid in the inversion---e.g., by using KFAC~\cite{martens2015optimizing}.
            It is possible to approximate the Hessian with the Gauss-Newton matrix or Fisher Information matrix~\citep{larsen1996design}.
            Various works use an identity approximation to the inverse, which is equivalent to \num{1}-step unrolled differentiation~\citep{luketina2016scalable, ren2018learning, balaji2018metareg, liu2018darts, finn2017model, shavitt2018regularization, nichol2018first}.
        
        \textbf{Unrolled differentiation for \ho\!\!.}
            A key difficulty in nested optimization is approximating how the optimized inner parameters (i.e., \nn weights) change with respect to the outer parameters (i.e., hyperparameters).
            We often optimize the inner parameters with gradient descent, so we can simply differentiate through this optimization.
            Differentiation through optimization has been applied to nested optimization problems by \cite{domke2012generic}, was scaled to \ho for \nns by \cite{maclaurin2015gradient}, and has been applied to various applications like learning optimizers~\cite{andrychowicz2016learning}.
            \cite{franceschi2018bilevel} provides convergence results for this class of algorithms, while \cite{franceschi2017forward} discusses forward- and reverse-mode variants.
            
            As the number of gradient steps we backpropagate through increases, so does the memory and computational cost.
            Often, gradient descent does not exactly minimize our objective after a finite number of steps---it only approaches a local minimum.
            Thus, to see how the hyperparameters affect the local minima, we may have to unroll the optimization infeasibly far.
            Unrolling a small number of steps can be crucial for performance but may induce bias~\cite{wu2018understanding}.
            \cite{shaban2018truncated} discusses connections between unrolling and the IFT, and proposes to unroll only the last $L$-steps.
            DrMAD~\cite{fu2016drmad} proposes an interpolation scheme to save memory.
        
        We compare hypergradient approximations in Table~\ref{tab:response_approxes}, and memory costs of gradient-based HO methods in Table~\ref{tab:hyper_methods}.
        We survey gradient-free \ho in Appendix~\ref{app:related_work}.
    
    \vspace{0.2cm}
    \section{Method}\label{sec:theory}
    \vspace{-0.2cm}
        In this section, we discuss how \ho is a uniquely challenging nested optimization problem and how to combine the benefits of the IFT and unrolled differentiation.
        
        \vspace{-0.1cm}
        \subsection{Hyperparameter Opt. is Pure-Response}
        \vspace{-0.1cm}
            Eq.~\ref{eqn:gradient-decomp} shows that the hypergradient decomposes into a \textit{\color{mydarkgreen}direct} and \textit{indirect gradient}.
            The bottleneck in hypergradient computation is usually finding the indirect gradient because we must take into account how the optimized parameters vary with respect to the hyperparameters.
            A simple optimization approach is to neglect the indirect gradient and only use the {\color{mydarkgreen}direct gradient}.
            This can be useful in zero-sum games like GANs~\cite{goodfellow2014generative} because they always have a non-zero direct term.
            
            However, using only the {\color{mydarkgreen}direct gradient} does not work in general games~\cite{balduzzi2018mechanics}.
            In particular, it does not work for \ho because the {\color{mydarkgreen}direct gradient} is identically \num{0} when the hyperparameters $\hp$ can only influence the validation loss by changing the optimized weights $\ep^*(\hp)$.
            For example, if we use regularization like weight decay when computing the training loss, but not the validation loss, then the {\color{mydarkgreen}direct gradient} is always \num{0}.
            
            If the {\color{mydarkgreen}direct gradient} is identically \num{0}, we call the game \textit{pure-response}.
            Pure-response games are uniquely difficult nested optimization problems for gradient-based optimization because we cannot use simple algorithms that rely on the {\color{mydarkgreen}direct gradient} like simultaneous SGD.
            Thus, we must approximate the indirect gradient.
        
        \newcommand{\smallSpace}{\hspace{0.03\textwidth}}
        \newcommand{\medSpace}{\hspace{0.09\textwidth}}
        \newcommand{\epA}{\ep_{\! \infty}}
        \vspace{-0.2cm}
        \subsection{Unrolled Optimization and the IFT}
        \vspace{-0.2cm}
            Here, we discuss the relationship between the IFT and differentiation through optimization.
            Specifically, we (1) introduce the recurrence relation that arises when we unroll SGD optimization, (2) give a formula for the derivative of the recurrence, and (3) establish conditions for the recurrence to converge.
            Notably, we show that the fixed points of the recurrence recover the IFT solution.
            We use these results to motivate a computationally tractable approximation scheme to the IFT solution.
            We give proofs of all results in Appendix~\ref{app:proofs}.
            
            Unrolling SGD optimization---given an initialization $\ep_0$---gives us the recurrence:
            \begin{equation}\label{eqn:sgd_recurrence}
                \ep_{\!i\!+\!1}\!(\hp) \!=\! T\!(\!\hp,\! \ep_{i}\!) \!=\! \ep_{\!i}(\hp) \!-\! \alpha \epderiv{\!\Ltr\!(\!\hp,\! \ep_{\!i}(\hp)\!)\!\!}
            \end{equation}
            In our exposition, assume that $\alpha=1$.
            We provide a formula for the derivative of the recurrence, to show that it converges to the IFT under some conditions.
            \begin{lemma*}\label{lemma:unrolled_grad}
                Given the recurrence from unrolling SGD optimization in Eq.~\ref{eqn:sgd_recurrence}, we have:
                \textnormal{
                    \begin{equation}
                        \!\hpderiv{\ep_{i+\!1\!}} = \! -\!\sum_{j\leq i} \!\! \left[ \! \prod_{k<j} \! I \! - \! \left. \trainHess{\ep_{i-k}(\hp)} \right|_{\!\hp, \ep_{i-\!k\!}(\hp)} \! \right] \!\! \left. \trainMixed{\ep_{i-\!j\!}(\hp)} \right|_{\!\hp, \ep_{i-\!j\!}(\hp)}
                    \end{equation}
                }
            \end{lemma*}
            This recurrence converges to a fixed point if the transition Jacobian $\epderiv{T}$ is contractive, by the Banach Fixed-Point Theorem~\citep{banach1922operations}.
            Theorem~\ref{thm:unrolled_to_IFT} shows that the recurrence converges to the IFT if we start at locally optimal weights $\ep_0 \!=\! \ep^*\!(\hparam)$, and the transition Jacobian $\epderiv{T}$ is contractive.
            We leverage that if an operator $U$ is contractive, then the Neumann series $\sum_{i = 0}^{\infty}U^k \!=\! (\text{Id} \!-\! U)^{-1}$.
            \begin{thm}[Neumann-SGD]
                Given the recurrence from unrolling SGD optimization in Eq.~\ref{eqn:sgd_recurrence}, if \textnormal{$\ep_0 = \ep^*(\hp)$}:
                \textnormal{
                    \begin{equation}
                        \hpderiv{\ep_{i+1}} = -\left. {\color{magenta}\left( \sum_{j < i} \left[I - \trainHess{\ep(\hp)}\right]^j \right)} \trainMixed{\ep(\hp)} \right|_{\ep^*(\hp)}
                    \end{equation}
                }
                and if \textnormal{$I + \trainHess{\ep(\hp)}$} is contractive:
                \textnormal{
                    \begin{equation}
                        \lim_{i\to\infty} \hpderiv{\ep_{i+1}} = -\left. {\color{magenta}\left[ \trainHess{\ep(\hp)} \right]^{-1}} \trainMixed{\ep(\hp)} \right|_{\ep^*(\hp)}
                    \end{equation}
                }
                \label{thm:unrolled_to_IFT}
            \end{thm}
            This result is also shown in~\cite{shaban2018truncated}, but they use a different approximation for computing the hypergradient---see Table~\ref{tab:response_approxes}.
            Instead, we use the following best-response Jacobian approximation, where $i$ controls the trade-off between computation and error bounds:
            \begin{align}\label{eqn:inv_approx}
                \hpderiv{\ep^*} 
                &\approx \left. -{\color{magenta}\left( \sum_{j < i} \left[I - 
                \alpha \trainHess{\ep(\hp)}\right]^j \right)} \trainMixed{\ep(\hp)} \right|_{\ep^*(\hp)}
            \end{align}
            \citet{shaban2018truncated} use an approximation that scales memory linearly in $i$, while ours is constant.
            We save memory because we reuse last $\ep$ $i$ times, while \citep{shaban2018truncated} needs the last $i$ $\ep$'s.
            Scaling the Hessian by the learning rate $\alpha$ is key for convergence.
            Our algorithm has the following main advantages relative to other approaches:
            \begin{itemize}
                \item It requires a constant amount of memory, unlike other unrolled differentiation methods~\cite{maclaurin2015gradient, shaban2018truncated}.
                \item It is more stable than conjugate gradient, like unrolled differentiation methods~\cite{liao2018reviving, shaban2018truncated}.
            \end{itemize}

        \subsection{Scope and Limitations}
            The assumptions necessary to apply the IFT are as follows:
            (1) $\Lval: \hdom \!\times\! \edom \!\to\! \Real$ is differentiable, 
            (2) $\Ltr: \hdom \!\times\! \edom \!\to\! \Real$ is twice differentiable with an invertible Hessian at $\response$, and
            (3) $\ep^*: \hdom \!\to\!\edom$ is differentiable.
            
            We need continuous hyperparameters to use gradient-based optimization, but many discrete hyperparameters (e.g., number of hidden units) have continuous relaxations~\citep{maddison2016concrete, jang2016categorical}.
            Also, we can only optimize hyperparameters that change the loss manifold, so our approach is not straightforwardly applicable to optimizer hyperparameters.

            To exactly compute hypergradients, we must find $(\hp',\!\ep')$ s.t.$\left. \smash{\trainGrad}\!\right|_{\hp'\!,\! \ep'}\!=\!0$, which we can only solve to a tolerance with an approximate solution denoted $\smash{\widehat{\ep^*}(\hp)}$.
            \cite{pedregosa2016hyperparameter} shows results for error in $\ep^*$ and the inversion.

    \vspace{-0.1cm}    
    \section{Experiments}\label{sec:experiments}
    \vspace{-0.3cm}
        We first compare the properties of Neumann inverse approximations and conjugate gradient, with experiments similar to \cite{liao2018reviving, maclaurin2015gradient, shaban2018truncated, pedregosa2016hyperparameter}.
        Then we demonstrate that our proposed approach can {overfit} the validation data with small training and validation sets.
        Finally, we apply our approach to high-dimensional \ho tasks: (1) {dataset distillation}; (2) learning a data augmentation network; and (3) tuning regularization parameters for an LSTM language model.
        
        \ho algorithms that are not based on implicit differentiation or differentiation through optimization---such as ~\citep{jaderberg2017population, jamieson2016non, bergstra2012random, kumar2018parallel, li2016hyperband, snoek2012practical}---do not scale to the high-dimensional hyperparameters we use.
        Thus, we cannot sensibly compare to them for high-dimensional problems.
        
        \vspace{-0.1cm}
        \subsection{Approximate Inversion Algorithms}
        \vspace{-0.1cm}
            In Fig.~\ref{fig:inverse_error} we investigate how close various approximations are to the true inverse.
            We calculate the distance between the approximate hypergradient and the true hypergradient.
            We can only do this for small-scale problems because we need the exact inverse for the true hypergradient.
            Thus, we use a linear network on the Boston housing dataset~\citep{harrison1978hedonic}, which makes finding the best-response $\ep^*$ and inverse training Hessian feasible.
            
            We measure the cosine similarity, which tells us how accurate the direction is, and the $\ell_2$ (Euclidean) distance between the approximate and true hypergradients.
            The Neumann approximation does better than CG in cosine similarity if we take enough \ho steps, while CG always does better for $\ell_2$ distance.

            In Fig.~\ref{fig:hessian_visual} we show the inverse Hessian for a fully-connected 1-layer \nn on the Boston housing dataset.
            The true inverse Hessian has a dominant diagonal, motivating identity approximations, while using more Neumann terms yields structure closer to the true inverse.
            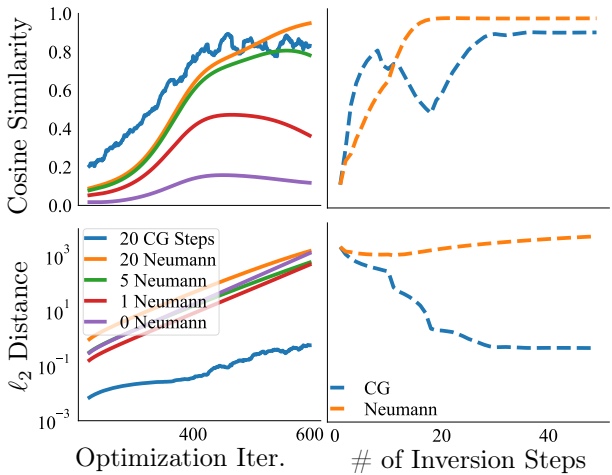
\begin{figure}[ht!]
                \vspace{-0.02\textheight}
                \centering
                \begin{tikzpicture}
                    \centering
                    \node (img){\includegraphics[trim={.25cm .25cm 0.25cm .25cm},clip, width=.4625\linewidth]{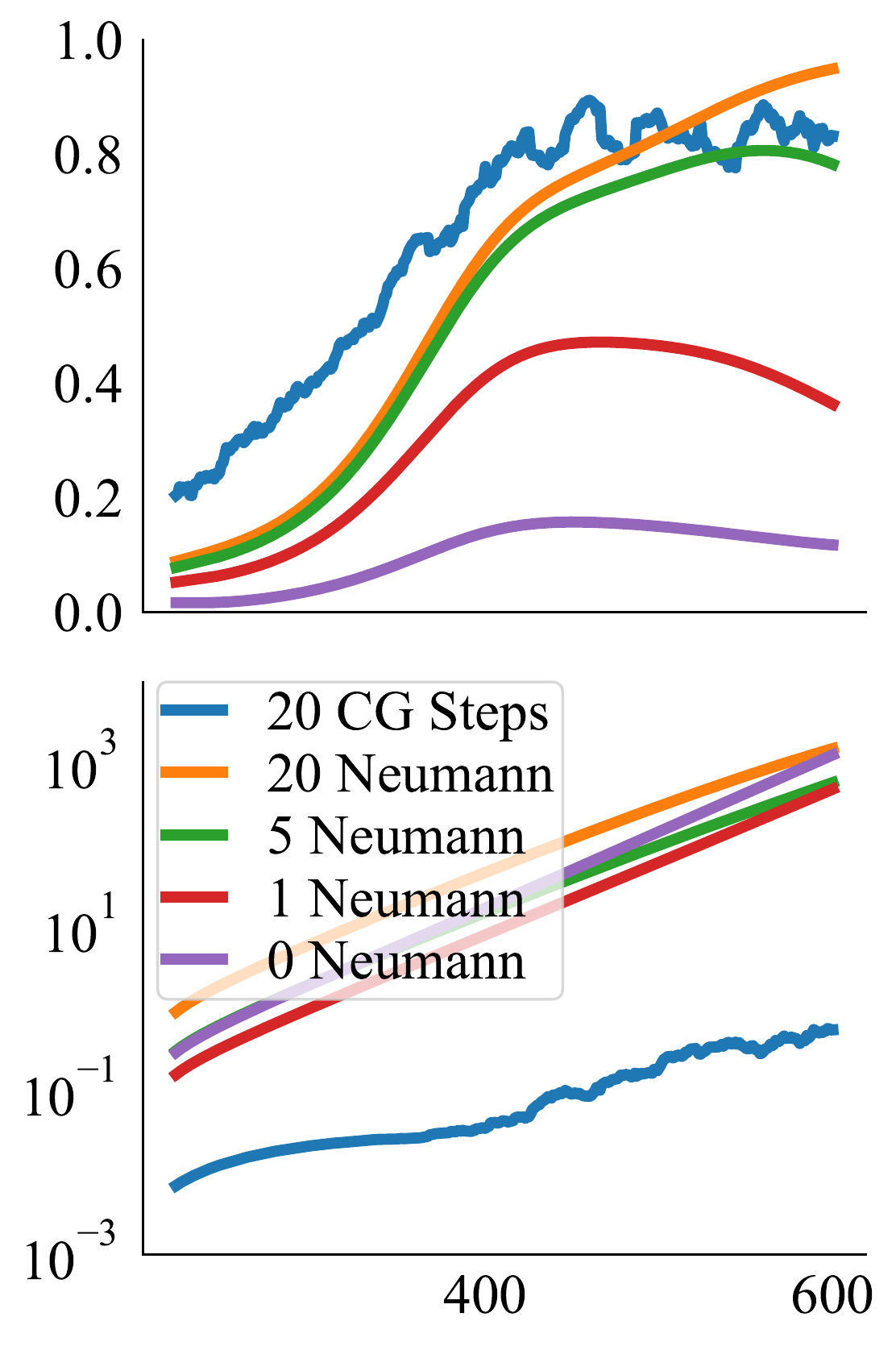}};
                    \node[left=of img, node distance=0cm, rotate=90, xshift=2.9cm, yshift=-.9cm, font=\color{black}] {Cosine Similarity};
                    \node[left=of img, node distance=0cm, rotate=90, xshift=-.35cm, yshift=-.9cm, font=\color{black}] {$\ell_2$ Distance};
                    \node[below=of img, node distance=0cm, yshift=1.2cm,font=\color{black}] {Optimization Iter.};
                    
                    \node (img2)[right=of img, xshift=-1.25cm]{\includegraphics[trim={.25cm .25cm 0.25cm .25cm},clip, width=.465\linewidth]{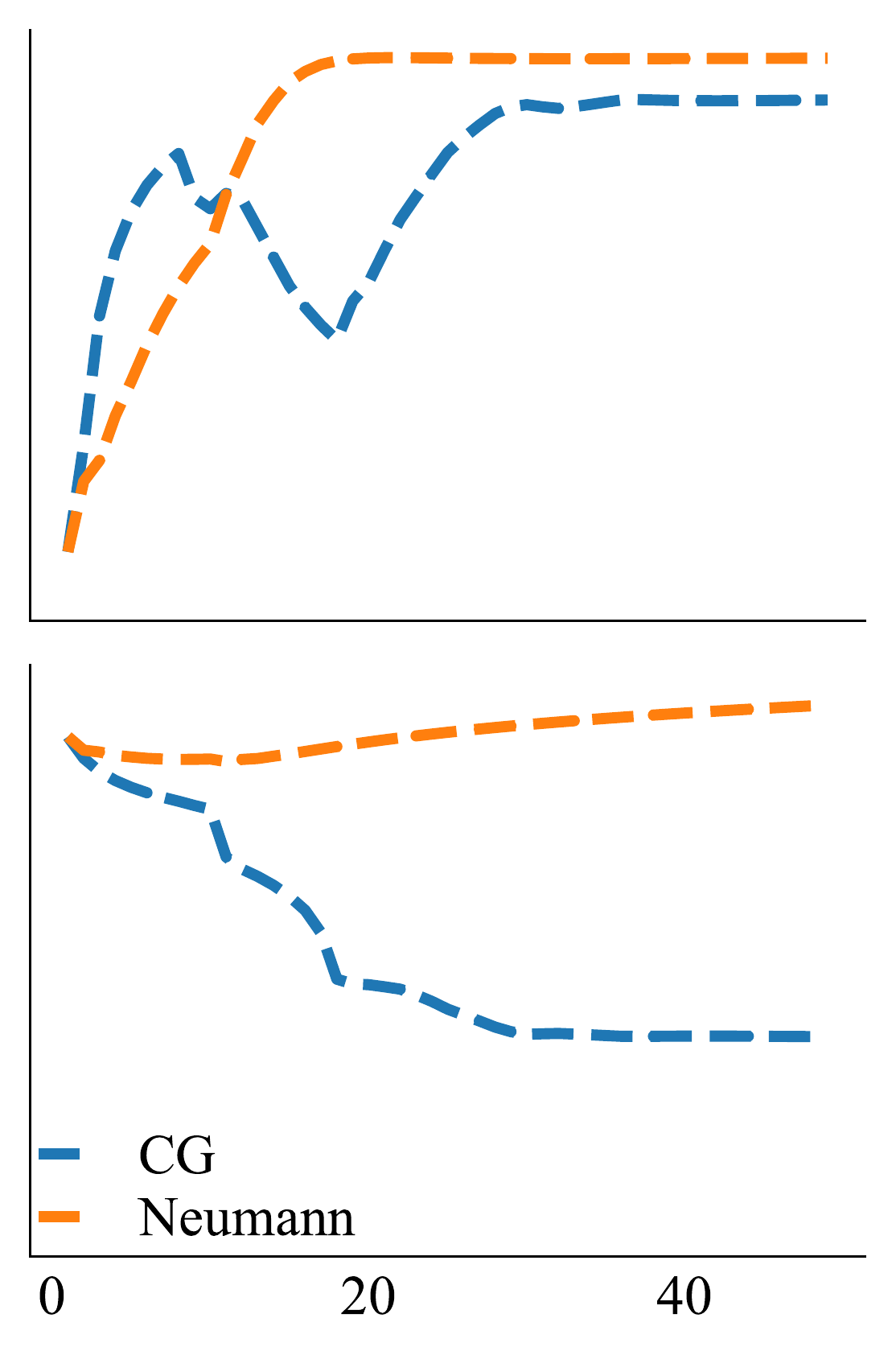}};
                    \node[below=of img2, node distance=0cm, yshift=1.2cm,font=\color{black}] {\# of Inversion Steps};
                \end{tikzpicture}
                \vspace{-0.035\textheight}
                \caption{
                    {Comparing approximate hypergradients for inverse Hessian approximations to true hypergradients.}
                    The Neumann scheme often has greater cosine similarity than CG, but larger $\ell_2$ distance for equal steps.
                }
                \label{fig:inverse_error}
                \vspace{-0.0125\textheight}
            \end{figure}
            \begin{figure}[ht!]
                \vspace{0.1cm}
                \centering
                \begin{tikzpicture}
                    \centering
                    \node[draw,inner sep=1pt] (img){\includegraphics[trim={3.5cm 1.4cm 4.2cm 2.0cm},clip, width=.285\linewidth]{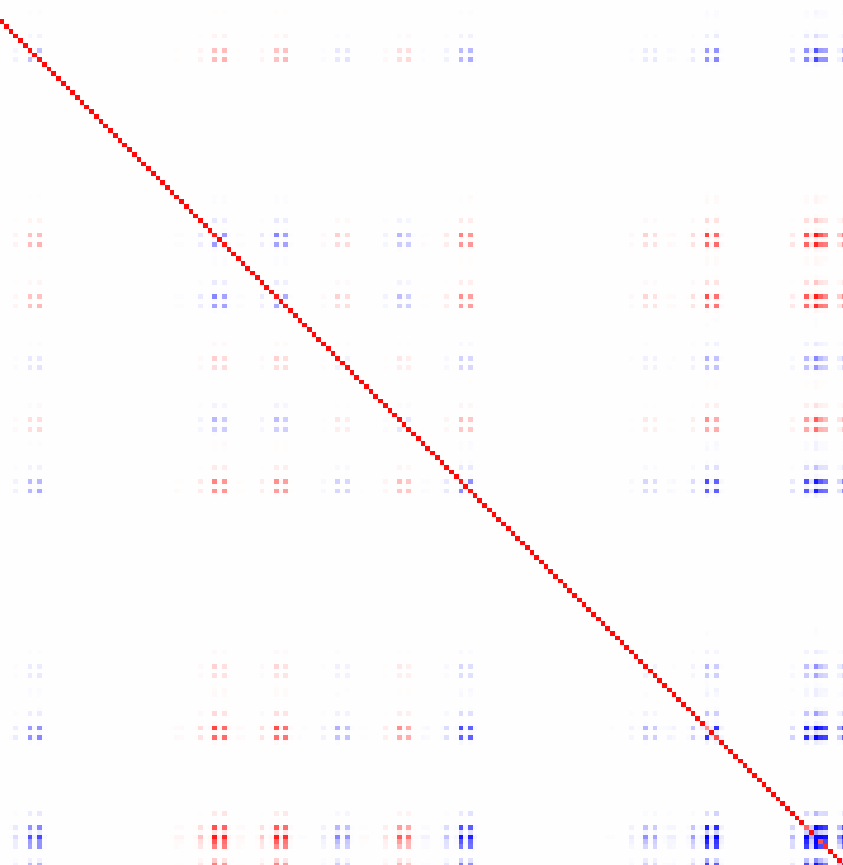}};
                    \node[below=of img, node distance=0cm, yshift=1.0cm, xshift=-.1cm,font=\color{black}] {1 Neumann};
                    
                    \node[draw,inner sep=1pt] (img2)[right=of img, xshift=-.98cm]{\includegraphics[trim={3.5cm 1.4cm 4.20cm 2.0cm},clip, width=.285\linewidth]{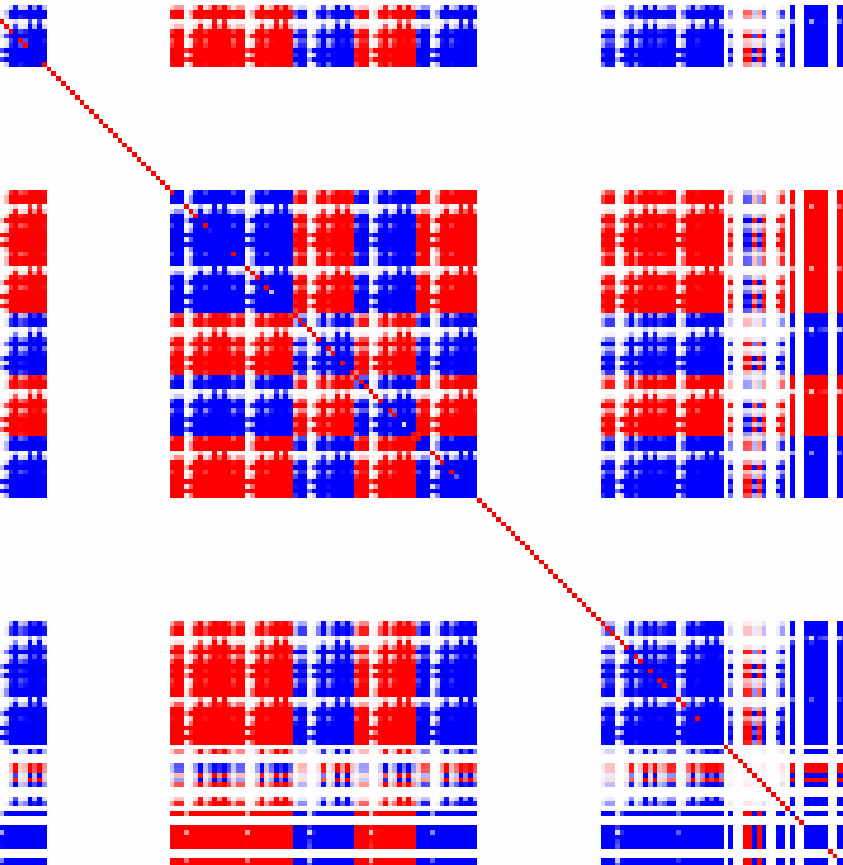}};
                    \node[below=of img2, node distance=0cm, yshift=1.cm,xshift=.1cm,font=\color{black}] {5 Neumann};
                    
                    \node[draw,inner sep=1pt] (img3)[right=of img2, xshift=-.98cm]{\includegraphics[trim={3.5cm 1.4cm 4.20cm 2.0cm},clip, width=.285\linewidth]{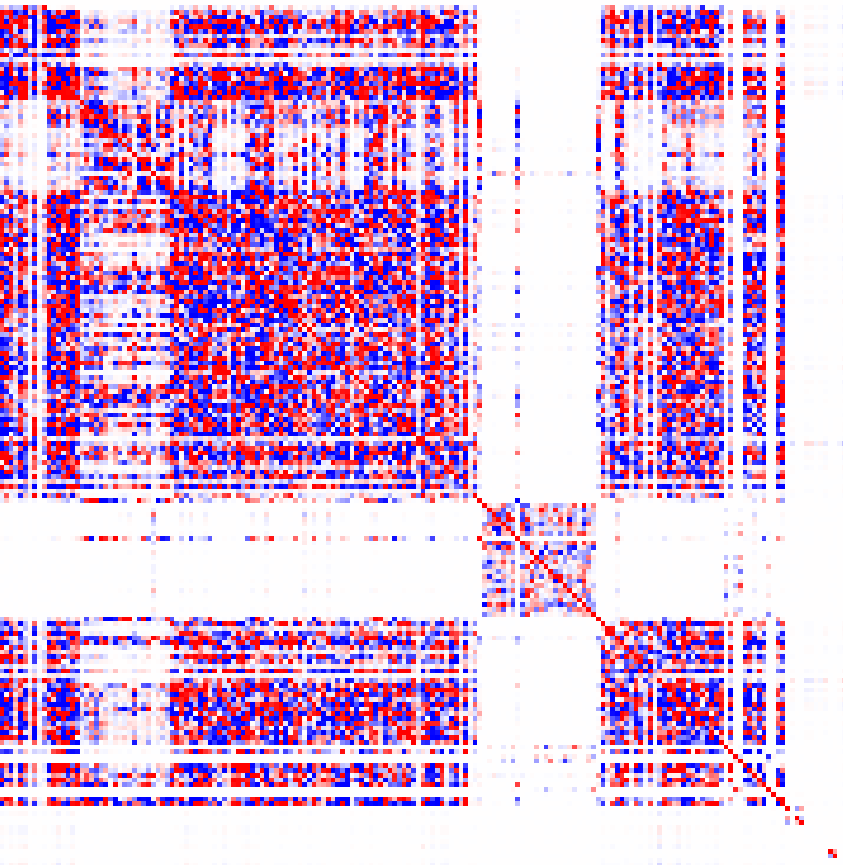}};
                    \node[below=of img3, node distance=0cm, yshift=1.cm,font=\color{black}] {True Inverse};
                    
                    \node (img4)[right=of img3, xshift=-1.1cm, yshift=-.025cm]{\includegraphics[trim={12.7cm 1.cm 1.05cm 1.25cm},clip, width=.082\linewidth]{hessian_true_inv.pdf}};
                \end{tikzpicture}
                \vspace{-0.035\textheight}
                \caption{
                    {Inverse Hessian approximations} {preprocessed by applying $\texttt{tanh}$} for a 1-layer, fully-connected \nn on the Boston housing dataset as in \cite{zhang2017noisy}.
                }
                \label{fig:hessian_visual}
            \end{figure}
        
        \newcommand{\numOverfitTrain}{\num{50}}
        \newcommand{\numOverfitValid}{\num{50}}
        \vspace{0.2cm}
        \subsection{Overfitting a Small Validation Set}
            In Fig.~\ref{fig:overfit_valid}, we check the capacity of our \ho algorithm to overfit the validation dataset.
            We use the same restricted dataset as in \cite{franceschi2017forward, franceschi2018bilevel} of $\numOverfitTrain$ training and validation examples, which allows us to assess \ho performance easily.
            We tune a separate weight decay hyperparameter for each \nn parameter as in \cite{balaji2018metareg, maclaurin2015gradient}.
            We show the performance with a linear classifier, AlexNet~\cite{krizhevsky2012imagenet}, and ResNet\num{44}~\cite{he2016deep}.
            For AlexNet, this yields more than $\num{50000000}$ hyperparameters, so we can perfectly classify our validation data by optimizing the hyperparameters.
            
            Algorithm~\ref{alg:joint_train} achieves \SI{100}{\percent} accuracy on the training and validation sets with significantly lower accuracy on the test set (Appendix~\ref{app:experiments}, Fig.~\ref{fig:overfit_valid_all}), showing that we have a powerful \ho algorithm.
            The same optimizer is used for weights and hyperparameters in all cases.
            %
            \begin{figure}[ht!]
                \vspace{-0.02\textheight}
                \centering
                \begin{tikzpicture}
                    \centering
                    \node (img){\includegraphics[trim={0.15cm 0cm 0.15cm 0cm},clip, width=.725\linewidth]{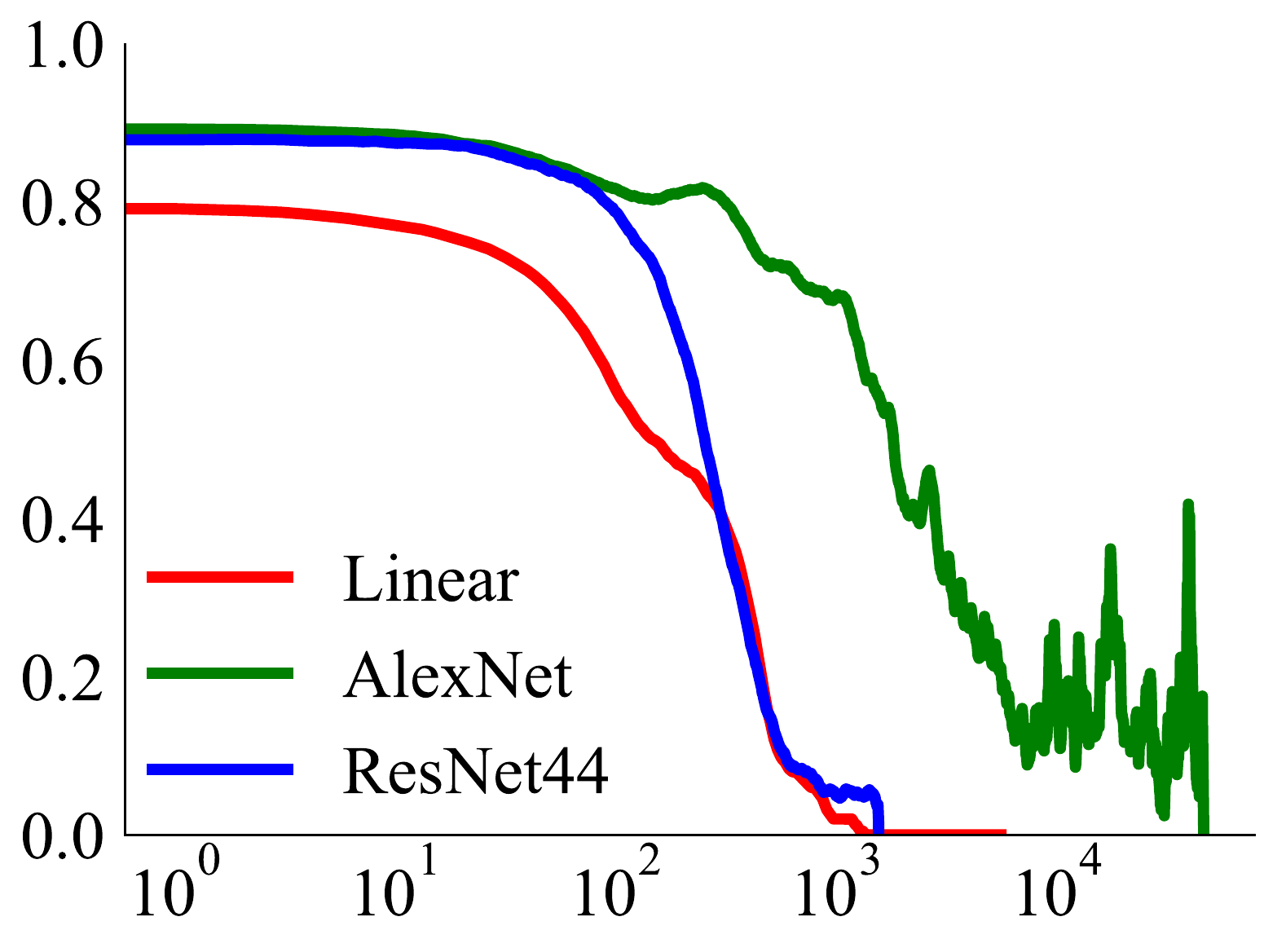}};
                    \node[left=of img, node distance=0cm, rotate=90, xshift=1.75cm, yshift=-.90cm, font=\color{black}] {Validation Error};
                    \node[below=of img, node distance=0cm, yshift=1.25cm,font=\color{black}] {Iteration};
                \end{tikzpicture}
                \vspace{-0.015\textheight}
                \caption{
                    {Algorithm~\ref{alg:joint_train} can overfit a small validation set on CIFAR-\num{10}.}
                    It optimizes for loss and achieves \SI{100}{\percent} validation accuracy for standard, large models. 
                }
                \label{fig:overfit_valid}
                \vspace{-0.015\textheight}
            \end{figure}
        
        \newcommand{\numDistillValid}{300}
        \subsection{Dataset Distillation}
            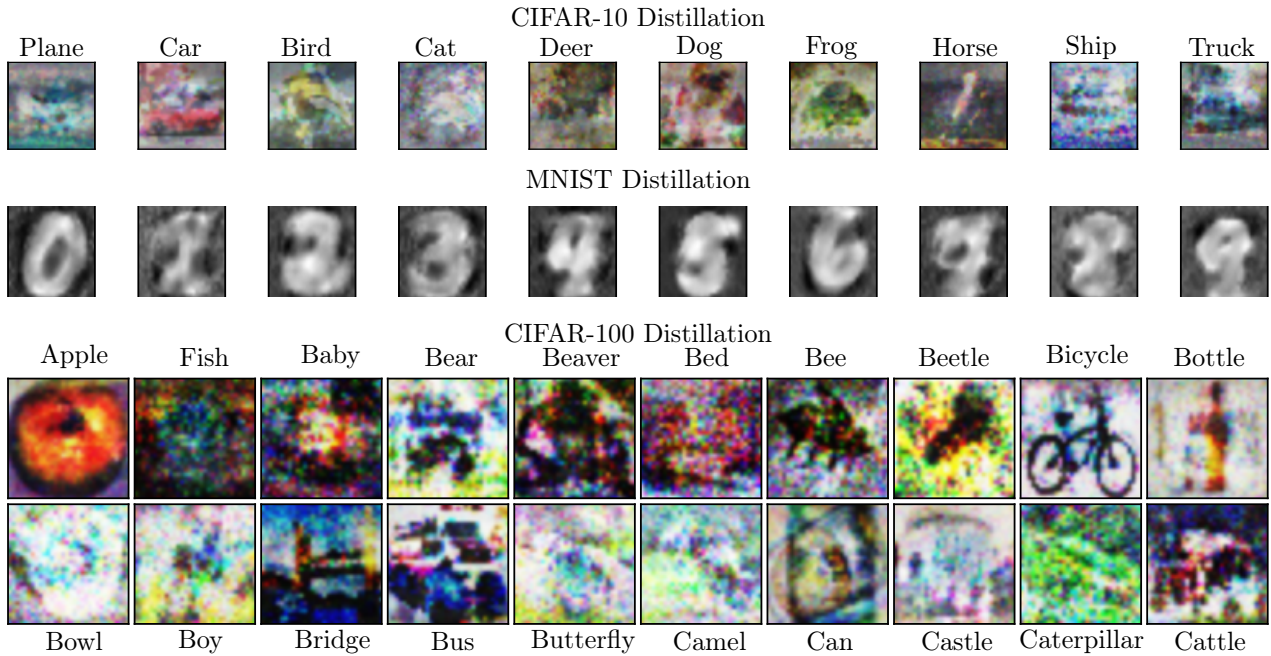
\begin{figure*}[h!]
                \vspace{-0.015\textheight}
                \centering
                \begin{tikzpicture}
                    \centering
                    \node (img){\includegraphics[trim={3.15cm 0cm 2.5cm 0cm},clip, width=.98\linewidth,height=1.5cm]{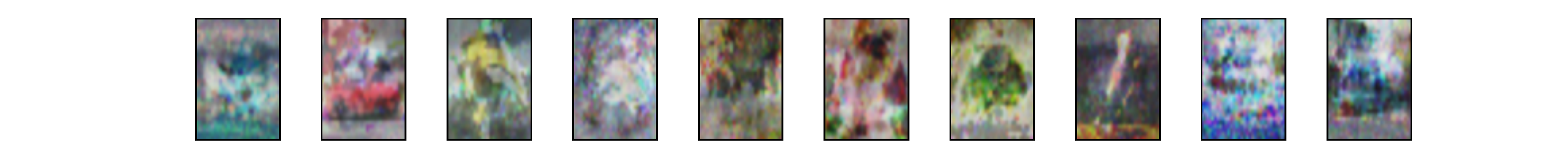}};
                    \node[above=of img, node distance=0cm, yshift=-.95cm,font=\color{black}] {{CIFAR-10 Distillation}};
                    \node(img_c1)[above=of img, node distance=0cm, yshift=-1.35cm, xshift=-7.8cm, font=\color{black}] {Plane};
                    \node(img_c2)[right=of img_c1, node distance=0cm, xshift=-.25cm, font=\color{black}] {Car};
                    \node(img_c3)[right=of img_c2, node distance=0cm, xshift=-.2cm, font=\color{black}] {Bird};
                    \node(img_c4)[right=of img_c3, node distance=0cm, xshift=-.15cm, font=\color{black}] {Cat};
                    \node(img_c5)[right=of img_c4, node distance=0cm, xshift=-.15cm, font=\color{black}] {Deer};
                    \node(img_c6)[right=of img_c5, node distance=0cm, xshift=-.15cm, font=\color{black}] {Dog};
                    \node(img_c7)[right=of img_c6, node distance=0cm, xshift=-.15cm, font=\color{black}] {Frog};
                    \node(img_c8)[right=of img_c7, node distance=0cm, xshift=-.25cm, font=\color{black}] {Horse};
                    \node(img_c9)[right=of img_c8, node distance=0cm, xshift=-.35cm, font=\color{black}] {Ship};
                    \node(img_c10)[right=of img_c9, node distance=0cm, xshift=-.3cm, font=\color{black}] {Truck};

                    \node (img2)[below=of img, yshift=.65cm]{\includegraphics[trim={3.15cm .5cm 2.5cm .5cm},clip, width=.98\linewidth,height=1.2cm]{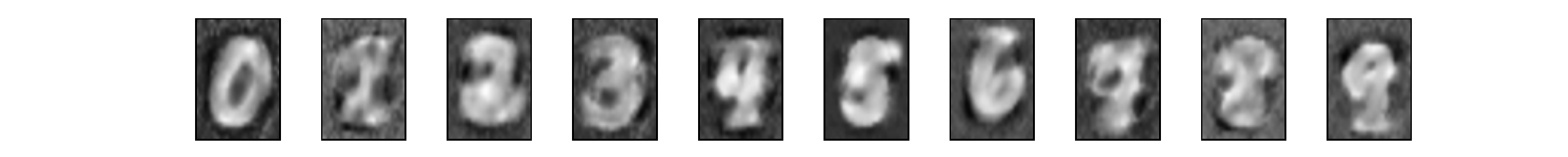}};
                    \node[above=of img2, node distance=0cm, yshift=-1.cm,font=\color{black}] {{MNIST Distillation}};
                    
                    \node (img3)[below=of img2, yshift=.2cm]{\includegraphics[trim={3.15cm 18.5cm 2.5cm 3cm},clip, width=.98\linewidth]{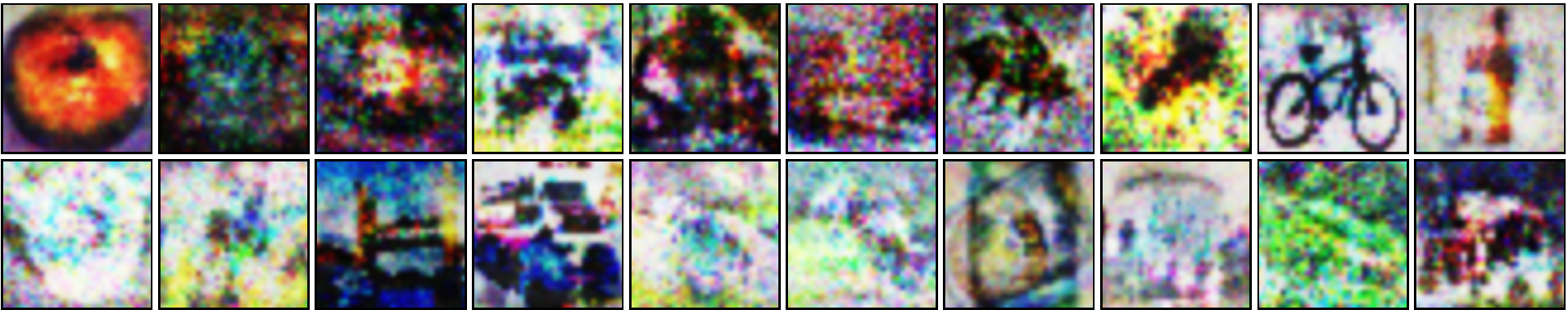}};
                    \node[above=of img3, node distance=0cm, yshift=-0.8cm,font=\color{black}] {{CIFAR-100 Distillation}};
                    \node(img_d1)[above=of img3, node distance=0cm, yshift=-1.15cm, xshift=-7.5cm, font=\color{black}] {Apple};
                    \node(img_d2)[right=of img_d1, node distance=0cm, xshift=-.3cm, font=\color{black}] {Fish};
                    \node(img_d3)[right=of img_d2, node distance=0cm, xshift=-.3cm, font=\color{black}] {Baby};
                    \node(img_d4)[right=of img_d3, node distance=0cm, xshift=-.4cm, font=\color{black}] {Bear};
                    \node(img_d5)[right=of img_d4, node distance=0cm, xshift=-.4cm, font=\color{black}] {Beaver};
                    \node(img_d6)[right=of img_d5, node distance=0cm, xshift=-.4cm, font=\color{black}] {Bed};
                    \node(img_d7)[right=of img_d6, node distance=0cm, xshift=-.25cm, font=\color{black}] {Bee};
                    \node(img_d8)[right=of img_d7, node distance=0cm, xshift=-.3cm, font=\color{black}] {Beetle};
                    \node(img_d9)[right=of img_d8, node distance=0cm, xshift=-.5cm, font=\color{black}] {Bicycle};
                    \node(img_d10)[right=of img_d9, node distance=0cm, xshift=-.65cm, font=\color{black}] {Bottle};
                    
                    \node(img_e1)[below=of img3, node distance=0cm, yshift=1.15cm, xshift=-7.5cm, font=\color{black}] {Bowl};
                    \node(img_e2)[right=of img_e1, node distance=0cm, xshift=-.25cm, font=\color{black}] {Boy};
                    \node(img_e3)[right=of img_e2, node distance=0cm, xshift=-.3cm, font=\color{black}] {Bridge};
                    \node(img_e4)[right=of img_e3, node distance=0cm, xshift=-.45cm, font=\color{black}] {Bus};
                    \node(img_e5)[right=of img_e4, node distance=0cm, xshift=-.5cm, font=\color{black}] {Butterfly};
                    \node(img_e6)[right=of img_e5, node distance=0cm, xshift=-.75cm, font=\color{black}] {Camel};
                    \node(img_e7)[right=of img_e6, node distance=0cm, xshift=-.45cm, font=\color{black}] {Can};
                    \node(img_e8)[right=of img_e7, node distance=0cm, xshift=-.35cm, font=\color{black}] {Castle};
                    \node(img_e9)[right=of img_e8, node distance=0cm, xshift=-.9cm, font=\color{black}] {Caterpillar};
                    \node(img_e10)[right=of img_e9, node distance=0cm, xshift=-.85cm, font=\color{black}] {Cattle};
                \end{tikzpicture}
                \vspace{-0.015\textheight}
                \caption{
                    {Distilled datasets for CIFAR-10, MNIST, and CIFAR-100.}
                    For CIFAR-100, we show the first 20 classes---the rest are in Appendix Fig.~\ref{app:cifar-100-distillation}.
                    We learn one distilled image per class, so after training a logistic regression classifier on the distillation, it generalizes to the rest of the data.
                }
                \label{fig:distilled_cifar}
            \end{figure*}
            Dataset distillation~\cite{maclaurin2015gradient, wang2018dataset} aims to learn a small, synthetic training dataset from scratch, that condenses the knowledge contained in the original full-sized training set.
            The goal is that a model trained on the synthetic data generalizes to the original validation and test sets.
            Distillation is an interesting benchmark for HO as it allows us to introduce tens of thousands of hyperparameters, and visually inspect what is learned: here, every pixel value in each synthetic training example is a hyperparameter.
            We distill MNIST and CIFAR-\num{10}/\num{100}~\cite{krizhevsky2009learning}, yielding $\num{28} \!\times\! \num{28} \!\times\! \num{10} = \num{7840}$, $\num{32} \!\times\! \num{32} \!\times\! \num{3} \!\times\! \num{10} = \num{30720}$, and $\num{32} \!\times\! \num{32} \!\times\! \num{3} \!\times\! \num{100} = \num{300720}$ hyperparameters, respectively.
            For these experiments, all labeled data are in our validation set, while our distilled data are in the training set.
            We visualize the distilled images for each class in Fig.~\ref{fig:distilled_cifar}, recovering recognizable digits for MNIST and reasonable color averages for CIFAR-\num{10}/\num{100}.
        
        \newcommand{\augmentFinalValAcc}{\SI{94}{\percent}}
        \newcommand{\augmentFinalTestAcc}{\SI{94}{\percent}}
        \newcommand{\numUnetParam}{\num{6659}}
        \vspace{-0.1cm}
        \subsection{Learned Data Augmentation}\label{sec:exp_data_augment}
        \vspace{-0.2cm}
            Data augmentation is a simple way to introduce invariances to a model---such as scale or contrast invariance---that improve generalization~\cite{cubuk2018autoaugment,xie2019unsupervised}.
            Taking advantage of the ability to optimize many hyperparameters, we learn data augmentation {from scratch} (Fig.~\ref{fig:augment}).
            
            Specifically, we learn a data augmentation network $\mathbf{\tilde{x}} = \mathbf{f}_{\hp}(\mathbf{x}, \epsilon)$ that takes a training example $\mathbf{x}$ and noise $\epsilon \sim \mathcal{N}(0, \mathbf{I})$, and outputs an augmented example $\mathbf{\tilde{x}}$.
            The noise $\epsilon$ allows us to learn {stochastic} augmentations.
            We parameterize $\mathbf{f}$ as a U-net~\citep{ronneberger2015u} with a residual connection from the input to the output, to make it easy to learn the identity mapping.
            The parameters of the U-net, $\hp$, are hyperparameters tuned for the validation loss---thus, we have $\numUnetParam$ hyperparameters.
            We trained a ResNet\num{18}~\cite{he2016deep} on CIFAR-10 with augmented examples produced by the U-net (that is simultaneously trained on the validation set).

            Results for the identity and Neumann inverse approximations are shown in Table~\ref{tab:neumann}.
            We omit CG because it performed no better than the identity.
            We found that using the data augmentation network improves validation and test accuracy by 2-3\%, and yields smaller variance between multiple random restarts.
            In \citep{mounsaveng2019adversarial}, a different augmentation network architecture is learned with adversarial training.
            \newcommand{\augWidth}{.113\textwidth}
            \newcommand{\augXShift}{-0.0675\textwidth}
            \newcommand{\augYShift}{0.0675\textwidth}
            \begin{figure}[h]
                \vspace{-0.1cm}
                \centering
                \begin{tikzpicture}
                    \centering
                    \node (img11){\includegraphics[trim={0cm 0cm 0cm 0cm},clip, width=\augWidth,height=1.8cm]{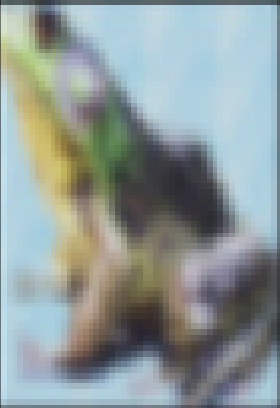}};
                    \node (img12)[right=of img11, xshift=\augXShift]{\includegraphics[trim={0cm 0cm 0cm 0cm},clip, width=\augWidth,height=1.8cm]{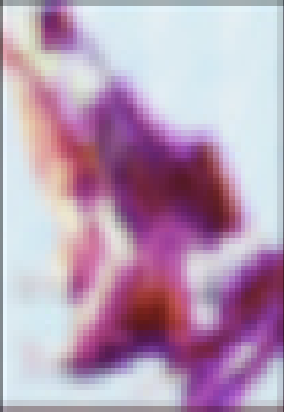}};
                    \node (img13)[right=of img12, xshift=\augXShift]{\includegraphics[trim={0cm 0cm 0cm 0cm},clip, width=\augWidth,height=1.8cm]{{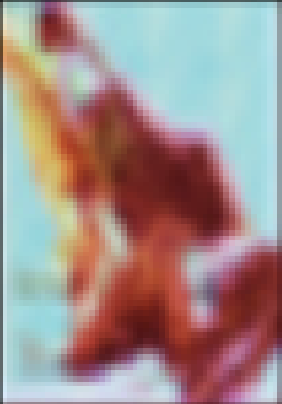}}};
                    \node (img14)[right=of img13, xshift=\augXShift]{\includegraphics[trim={0cm 0cm 0cm 0cm},clip, width=\augWidth,height=1.8cm]{{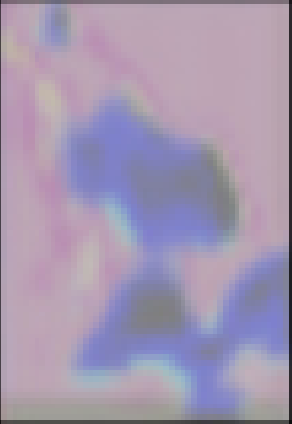}}};
                    
                    \node[above=of img11, node distance=0cm, yshift=-1.15cm,font=\color{black}] {{Original}};
                    \node[above=of img12, node distance=0cm, yshift=-1.15cm,font=\color{black}] {{Sample 1}};
                    \node[above=of img13, node distance=0cm, yshift=-1.15cm,font=\color{black}] {{Sample 2}};between
                    \node[above=of img14, node distance=0cm, yshift=-1.08cm,font=\color{black}] {{Pixel Std.}};
                    
                    \node (img21)[below=of img11, yshift=\augYShift]{\includegraphics[trim={0cm 0cm 0cm 0cm},clip, width=\augWidth,height=1.8cm]{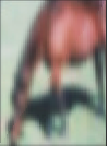}};
                    \node (img22)[right=of img21, xshift=\augXShift]{\includegraphics[trim={0cm 0cm 0cm 0cm},clip, width=\augWidth,height=1.8cm]{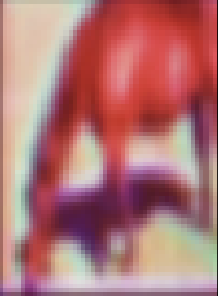}};
                    \node (img23)[right=of img22, xshift=\augXShift]{\includegraphics[trim={0cm 0cm 0cm 0cm},clip, width=\augWidth,height=1.8cm]{{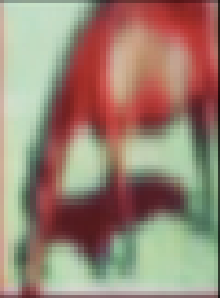}}};
                    \node (img24)[right=of img23, xshift=\augXShift]{\includegraphics[trim={0cm 0cm 0cm 0cm},clip, width=\augWidth,height=1.8cm]{{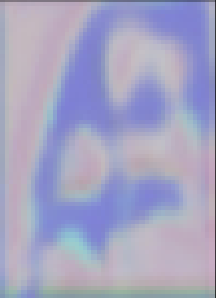}}};
                    
                    \node (img31)[below=of img21, yshift=\augYShift]{\includegraphics[trim={0cm 0cm 0cm 0cm},clip, width=\augWidth,height=1.8cm]{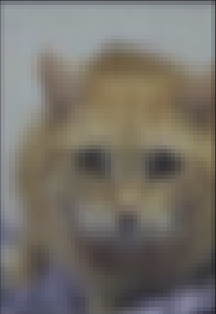}};
                    \node (img32)[right=of img31, xshift=\augXShift]{\includegraphics[trim={0cm 0cm 0cm 0cm},clip, width=\augWidth,height=1.8cm]{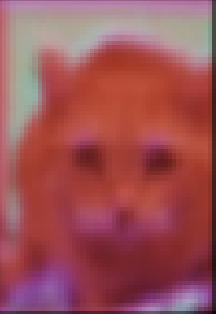}};
                    \node (img33)[right=of img32, xshift=\augXShift]{\includegraphics[trim={0cm 0cm 0cm 0cm},clip, width=\augWidth,height=1.8cm]{{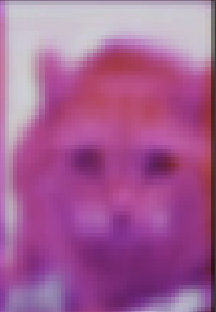}}};
                    \node (img34)[right=of img33, xshift=\augXShift]{\includegraphics[trim={0cm 0cm 0cm 0cm},clip, width=\augWidth,height=1.8cm]{{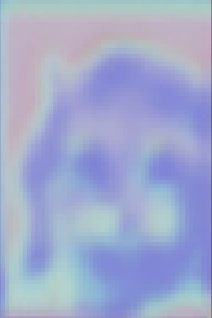}}};
                \end{tikzpicture}
                \vspace{-0.8cm}
                \caption{
                    {Learned data augmentations.}
                    The original image is on the left, followed by two augmented samples and the standard deviation of the pixel intensities from the augmentation distribution.
                }
                \label{fig:augment}
                \vspace{-0.15cm}
            \end{figure}
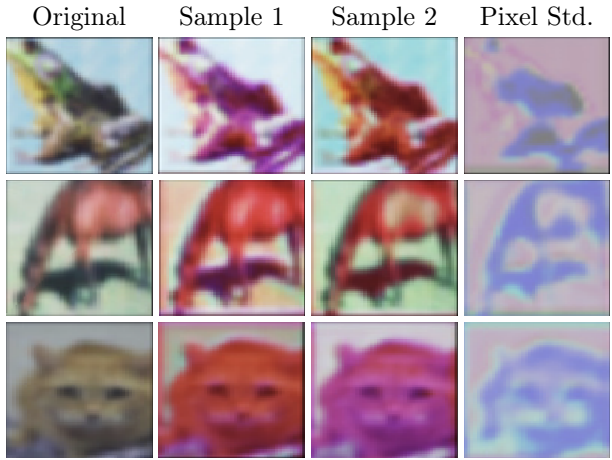
            
            \begin{table}[H]
            	\centering
            	\begin{tabular}{ccc}
            		{Inverse Approx.} & {Validation} & {Test}
            		\\
            		\midrule
            	    \num{0} 
            	        & 92.5 \rpm 0.021
            	        & 92.6 \rpm 0.017
            	        \\
            		\num{3} Neumann
            		    & \textbf{95.1} \rpm 0.002
            		    & 94.6 \rpm 0.001
            		    \\
            		\num{3} Unrolled Diff.
            		    & 95.0 \rpm 0.002
            		    & \textbf{94.7} \rpm 0.001 
            		    \\
            	    $I$
            		    & 94.6 \rpm 0.002
            		    & 94.1 \rpm 0.002
            		  \\
            	\end{tabular}
            	\vspace{-0.01\textheight}
                \caption{
                    {Accuracy of different inverse approximations.}
                    Using $\num{0}$ means that no \ho occurs, and the augmentation is initially the identity.
                    The Neumann approach performs similarly to unrolled differentiation~\citep{maclaurin2015gradient, shaban2018truncated} with equal steps and less memory.
                    Using more terms does better than the identity, and the identity performed better than CG (not shown), which was unstable.
                }
                \label{tab:neumann}
                \vspace{-0.015\textheight}
            \end{table}

    \vspace{0.8cm}
    \subsection{RNN Hyperparameter Optimization}
    \label{sec:exp_rnn}
        We also used our proposed algorithm to tune regularization hyperparameters for an LSTM~\cite{hochreiter1997long} trained on the Penn TreeBank (PTB) corpus~\cite{marcus1993building}.
        As in~\cite{gal2016theoretically}, we used a \num{2}-layer LSTM with \num{650} hidden units per layer and \num{650}-dimensional word embeddings.
        Additional details are provided in Appendix~\ref{app:exp_rnn}.
        
        \textbf{Overfitting Validation Data.}
            We first verify that our algorithm can overfit the validation set in a small-data setting with 10 training and 10 validation sequences (Fig.~\ref{fig:rnn-overfitting}).
            The LSTM architecture we use has \num{13280400} weights, and we tune a separate weight decay hyperparameter per weight.
            We overfit the validation set, reaching nearly 0 validation loss.
            \begin{figure}[h]
                \centering
                \vspace{-0.0cm}
                \begin{tikzpicture}
                    \centering
                    \node (img){\includegraphics[trim={0.7cm .8cm .55cm .55cm},clip, width=.6\linewidth]{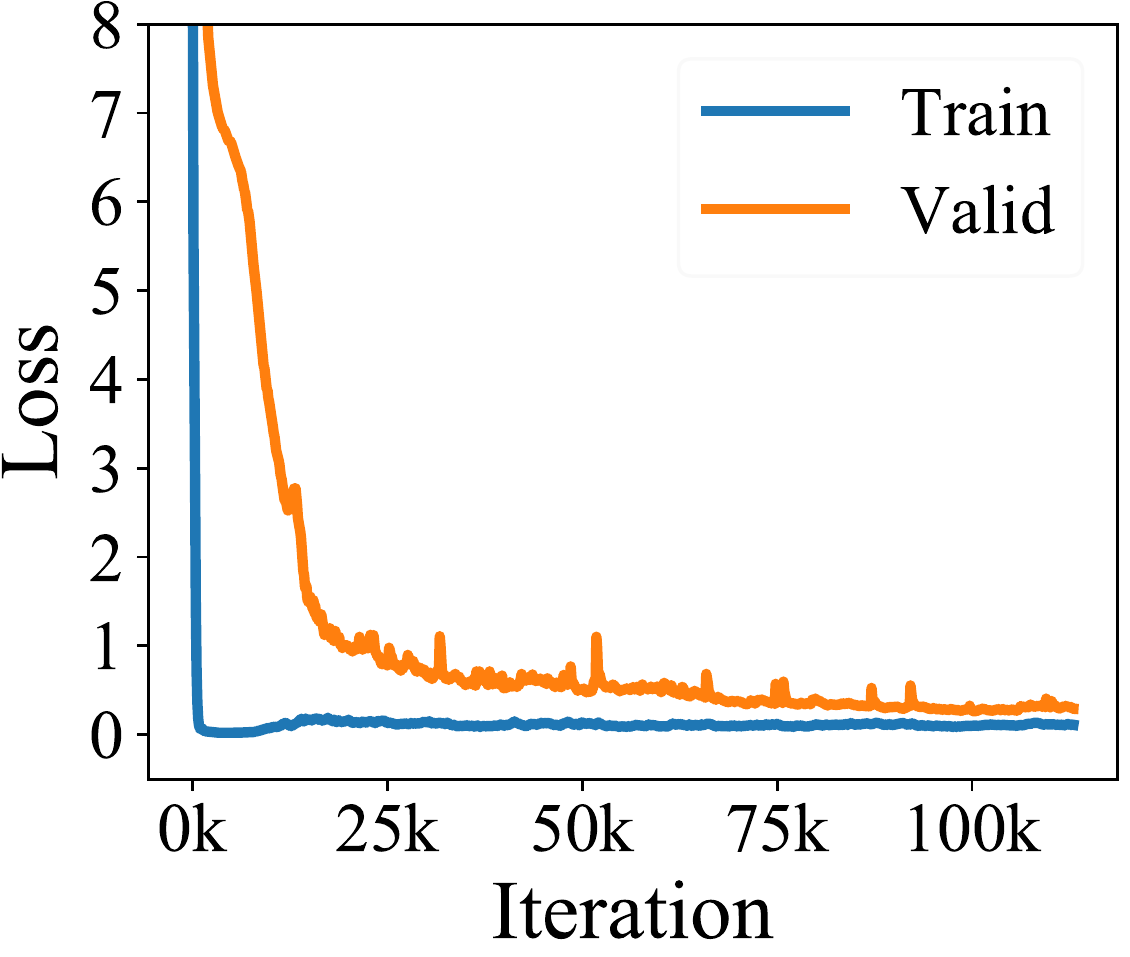}};
                    \node[left=of img, node distance=0cm, rotate=90, xshift=.5cm, yshift=-.90cm, font=\color{black}] {Loss};
                    \node[below=of img, node distance=0cm, yshift=1.25cm,font=\color{black}] {Iteration};
                \end{tikzpicture}
                \vspace{-0.2cm}
                \caption{
                    {Alg.~\ref{alg:joint_train} can overfit a small validation set with an LSTM on PTB.}
                }
                \label{fig:rnn-overfitting}
                \vspace{-0.05cm}
            \end{figure}

        \textbf{Large-Scale HO.}
            There are various forms of regularization used for training RNNs, including {variational dropout}~\citep{kingma2015variational} on the input, hidden state, and output; {embedding dropout} that sets rows of the embedding matrix to 0, removing tokens from all sequences in a mini-batch; DropConnect~\cite{wan2013regularization} on the hidden-to-hidden weights; and activation and temporal activation regularization.
            We tune these \num{7} hyperparameters simultaneously.
            Additionally, we experiment with tuning separate dropout/DropConnect rate for each activation/weight, giving \num{1691951} total hyperparameters.
            To allow for gradient-based optimization of dropout rates, we use concrete dropout~\cite{gal2017concrete}.
            
            Instead of using the small dropout initialization as in \cite{mackay2019self}, we use a larger initialization of $\num{0.5}$, which prevents early learning rate decay for our method.
            The results for our new initialization with no HO, our method tuning the same hyperparameters as \cite{mackay2019self} (``Ours''), and our method tuning many more hyperparameters (``Ours, Many'') are shown in Table~\ref{tab:rnn_compare}.
            We are able to tune hyperparameters more quickly and achieve better perplexities than the alternatives.
            \begin{table}[H]
                \vspace{-0.25cm}
            	\centering
            	\begin{tabular}{cccc}
            		\textbf{Method} & \textbf{Validation} & \textbf{Test} & \textbf{Time(s)}\\
            		\midrule
            	    Grid Search
            	        & 97.32
            	        & 94.58
            	        & 100k\\
            		Random Search
            		    & 84.81
            		    & 81.46
            		    & 100k\\
            	    Bayesian Opt.
            		    & 72.13
            		    & 69.29
            		    & 100k\\
            		STN
            		    & 70.30
            		    & 67.68
            		    & 25k\\
            		\hline
            		\textbf{No \ho}
            		    & 75.72
        		        & 71.91
        		        & 18.5k\\
        		    \textbf{Ours}
        		        & 69.22
        		        & 66.40
        		        & 18.5k\\
            		\textbf{Ours, Many}
            		    & \textbf{68.18}
            		    & \textbf{66.14}
            		    & 18.5k\\
                	\end{tabular}
                	\vspace{-0.345cm}
                    \caption{
                        {Comparing \HO methods for LSTM training on PTB.}
                        We tune millions of hyperparameters faster and with comparable memory to competitors tuning a handful.
                        Our method competitively optimizes the same \num{7} hyperparameters as baselines from \cite{mackay2019self} (first four rows).
                        We show a performance boost by tuning millions of hyperparameters, introduced with per-unit/weight dropout and DropConnect.
                        ``No HO'' shows how the hyperparameter initialization affects training.
                    }
                    \vspace{-0.5cm}
                    \label{tab:rnn_compare}
                \end{table}
    
    \subsection{Effects of Many Hyperparameters}
        \vspace{-0.15cm}
        Given the ability to tune high-dimensional hyperparameters and the potential risk of overfitting to the validation set, should we reconsider how our training and validation splits are structured?
        Do the same heuristics apply as for low-dimensional hyperparameters (e.g., use $\sim 10\%$ of the data for validation)?

        In Fig.~\ref{fig:valid_prop_comparison} we see how splitting our data into training and validation sets of different ratios affects test performance.
        We show the results of jointly optimizing the \nn weights and hyperparameters, as well as the results of fixing the final optimized hyperparameters and re-training the \nn weights from scratch, which is a common technique for boosting performance~\citep{Goodfellow-et-al-2016}.
        
        We evaluate a high-dimensional regime with a separate weight decay hyperparameter per \nn parameter, and a low-dimensional regime with a single, global weight decay.
        We observe that: 
        (1) for few hyperparameters, the optimal combination of validation data and hyperparameters has similar test performance with and without re-training, because the optimal amount of validation data is small; and
        (2) for many hyperparameters, the optimal combination of validation data and hyperparameters is significantly affected by re-training, because the optimal amount of validation data needs to be large to fit our hyperparameters effectively.
        
        For few hyperparameters, our results agree with the standard practice of using \num{10}\% of the data for validation and the other \num{90}\% for training.
        For many hyperparameters, our results show that we should use larger validation partitions for \ho\!\!.
        If we use a large validation partition to fit the hyperparameters, it is critical to re-train our model with all of the data.

        \begin{figure}[h!]
            \vspace{-0.125cm}
            \begin{tikzpicture}
                \centering
                \node (img){\includegraphics[trim={1.25cm 1.35cm .6cm 1.3cm},clip, width=.475\linewidth]{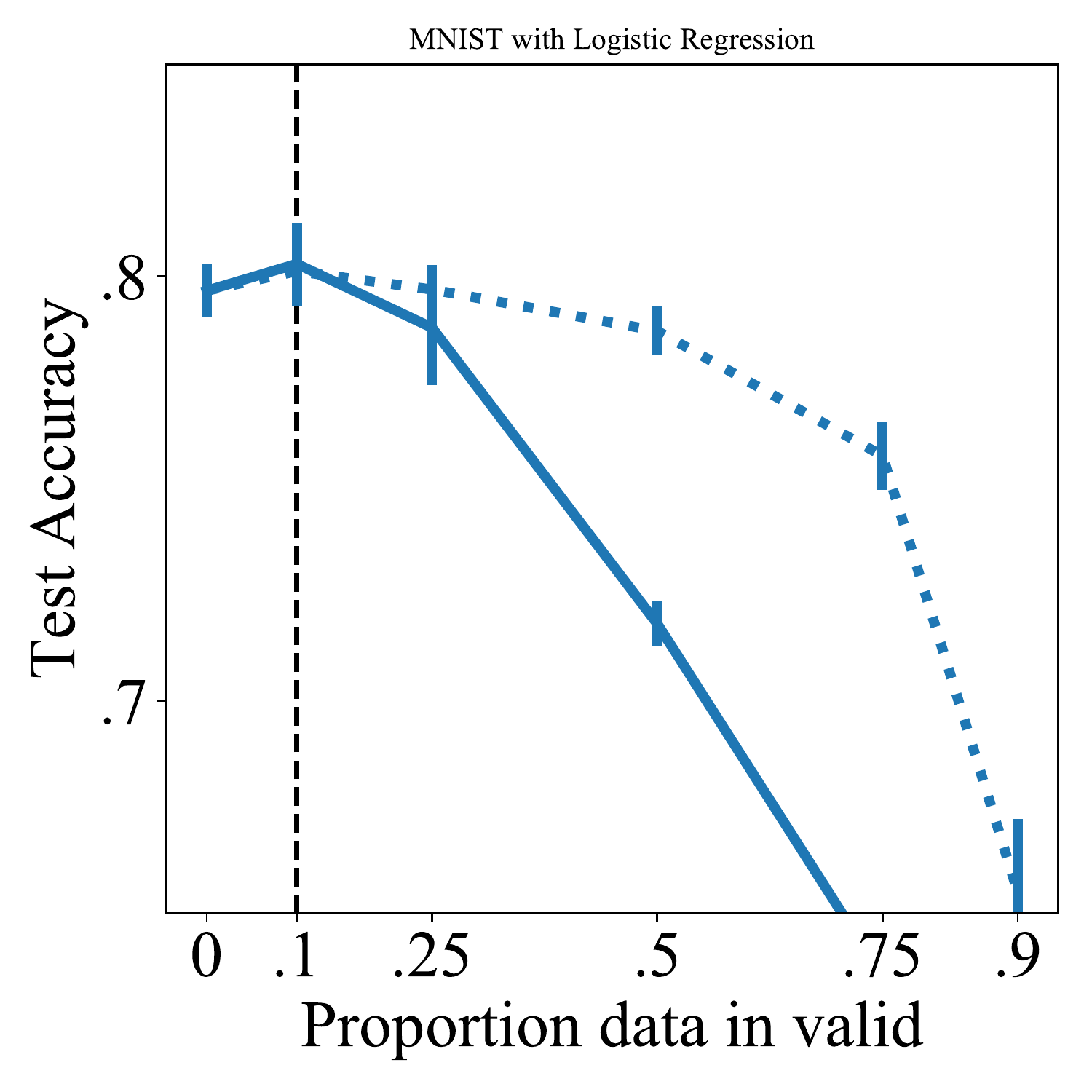}};
                \node[left=of img, node distance=0cm, rotate=90, xshift=1.5cm, yshift=-.95cm, font=\color{black}] {Test Accuracy};
                \node[above=of img, node distance=0cm, yshift=-1.15cm, xshift=.2cm, font=\color{black}] {{Without re-training}};
                \node[below=of img, node distance=0cm, yshift=1.25cm,font=\color{black}] {\% Data in Validation};
                
                \node (img2)[right=of img, xshift=-1.2cm]{\includegraphics[trim={2.3cm 1.35cm .6cm 1.3cm},clip, width=.44\linewidth]{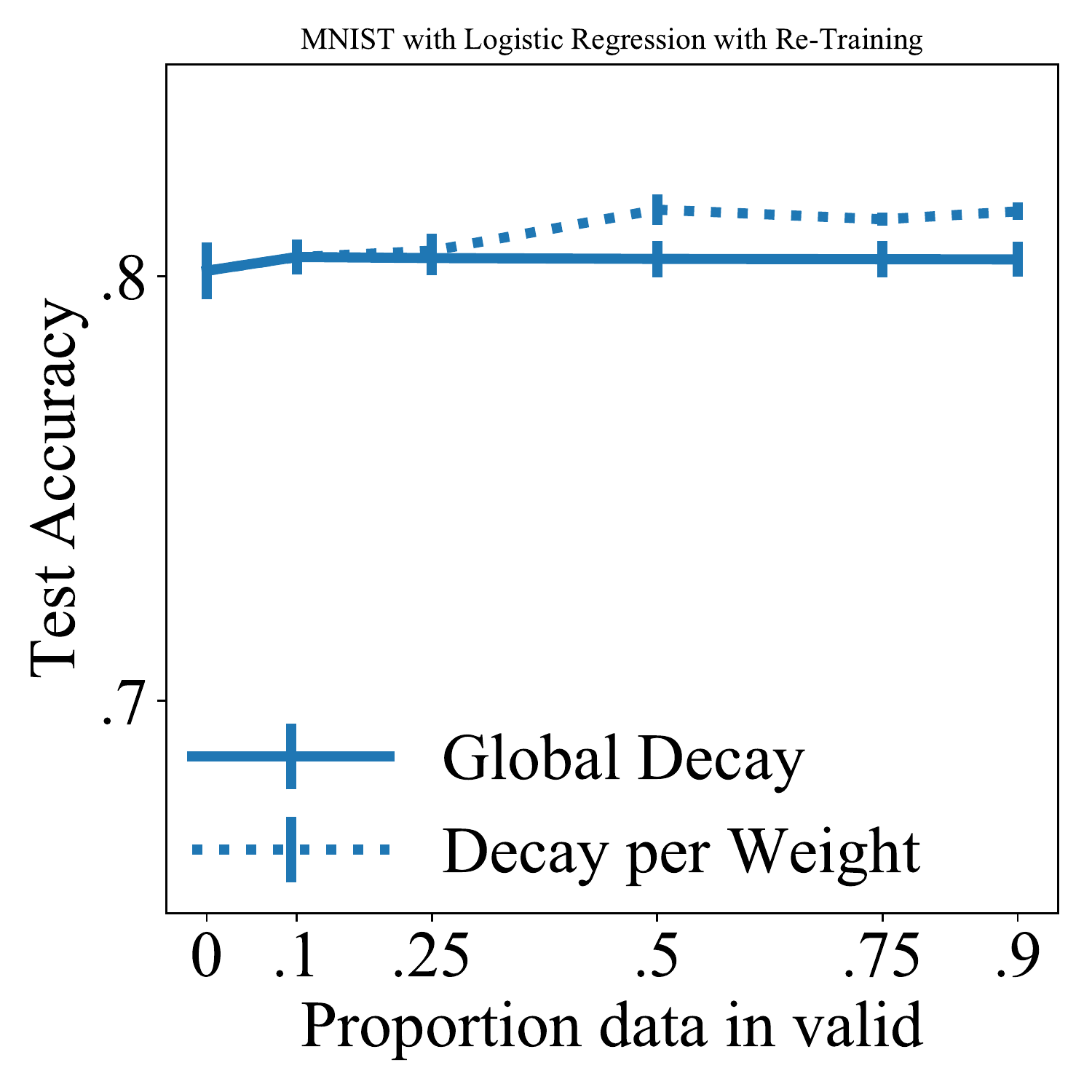}};
                \node[above=of img2, node distance=0cm, yshift=-1.15cm,font=\color{black}] {{With re-training}};
                \node[below=of img2, node distance=0cm, yshift=1.25cm,font=\color{black}] {\% Data in Validation};
            \end{tikzpicture}
            \vspace{-0.3cm}
            \caption{
                {Test accuracy of logistic regression on MNIST, with different size validation splits.}
                Solid lines correspond to a single global weight decay ($\num{1}$ hyperparameter), while dotted lines correspond to a separate weight decay per weight (many hyperparameters).
                The best validation proportion for test performance is different after re-training for many hyperparameters, but similar for few hyperparameters.
            }
            \label{fig:valid_prop_comparison}
            \vspace{-0.025\textheight}
        \end{figure}
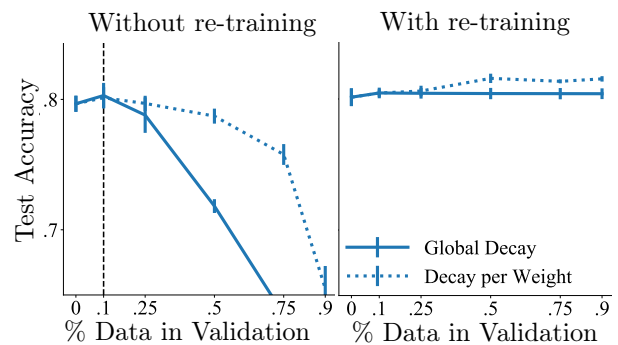

    \vspace{-0.1cm}
    \section{Conclusion}
    \vspace{-0.225cm}
        We present a gradient-based hyperparameter optimization algorithm that scales to high-dimensional hyperparameters for modern, deep \nns\!\!.
        We use the implicit function theorem to formulate the hypergradient as a matrix equation, whose bottleneck is inverting the Hessian of the training loss with respect to the \nn parameters.
        We scale the hypergradient computation to large \nns by approximately inverting the Hessian, leveraging a relationship with unrolled differentiation.
        
        We believe algorithms of this nature provide a path for practical nested optimization, where we have Hessians with known structure.
        Examples of this include GANs~\cite{goodfellow2014generative}, and other multi-agent games~\cite{foerster2018learning, letcher2018stable}.
        
    \section*{Acknowledgements}
        We thank Chris Pal for recommending we investigate re-training with all the data, Haoping Xu for discussing related experiments on inverse-approximation variants, Roger Grosse for guidance, and Cem Anil \& Chris Cremer for their feedback on the paper.
        Paul Vicol was supported by a JP Morgan AI Fellowship.
        We also thank everyone else at Vector for helpful discussions and feedback.
    
    \bibliography{references}
    \clearpage
    \newpage

    \appendix
        \twocolumn[
            \aistatstitle{Optimizing Millions of Hyperparameters by Implicit Differentiation\\
            Appendix}
            \vspace{-.75cm}
        ]
        
        \section{Extended Background}\label{app:background}
            In this section we provide an outline of our notation (Table~\ref{tab:TableOfNotation}), and the proposed algorithm.
            Here, we assume we have access to to a finite dataset $\dataset = \{(\vx_i, \vy_i) | i = 1 \dots n \}$, with $n$ examples drawn from the distribution $p(\vx, \vy)$ with support $\mathcal{P}$.
            We denote the input and target domains by $\mathcal{X}$ and $\mathcal{Y}$, respectively.
            Assume $\vy: \xdom \to \tdom$ is a function and we wish to learn $\hat{\vy}:\xdom \times \edom \to \tdom$ with a \nn parameterized by $\ep \in \edom$, s.t. $\hat{\vy}$ is close to $\vy$.
            We measure how close a predicted value is to a target with the prediction loss $\loss: \tdom \times \tdom \to \Real$.
            Our goal is to minimize the expected prediction loss or population risk: $\argmin_{\ep} \mathbb{E}_{\vx \sim p(\vx)}[\loss(\hat{\vy}(\vx, \ep), \vy(\vx))]$.
            Since we only have access to a finite number of samples, we minimize the empirical risk: $\argmin_{\ep} \nicefrac{1}{n}\sum_{\vx, \vy \in \dataset}\loss(\hat{\vy}(\vx, \ep), \vy(\vx))$.
            
            Due to a limited size dataset $\dataset$, there may be a significant difference between the minimizer of the empirical risk and the population risk.
            We can estimate this difference by partitioning our dataset into {training} and {validation} datasets--- $\dataset_{train}, \dataset_{valid}$.
            We find the minimizer over the training dataset $\dataset_{train}$, and estimate its performance on the population risk by evaluating the empirical risk over the validation dataset $\dataset_{valid}$.
            We introduce modifications to the empirical training risk to decrease our population risk, parameterized by $\hp \in \hdom$.
            These parameters for generalization are called the hyperparameters.
            We call the modified empirical training risk our training loss for simplicity and denote it $\Ltr (\hp, \ep)$.
            Our validation empirical risk is called validation loss for simplicity and denoted by $\Lval (\hp, \ep)$.
            Often the validation loss does not directly depend on the hyperparameters, and we just have $\Lval (\ep)$.

            The population risk is estimated by plugging the training loss minimizer $\response = \argmin_{\ep} \Ltr (\hp, \ep)$ into the validation loss for the estimated population risk $\Lval^*(\hp) = \LvalResponse$.
            We want our hyperparameters to minimize the estimated population risk: $\hp^* = \argmin_{\hp} \Lval^*(\hp)$.
            We can create a third partition of our dataset $\dataset_{test}$ to assess if we have overfit the validation dataset $\dataset_{valid}$ with our hyperparameters $\hp$.
            
        \section{Extended Related Work}
        \label{app:related_work}
            \textbf{Independent \ho\!\!:}
                A simple class of \ho algorithms involve making a number of independent hyperparameter selections, and training the model to completion on them.
                Popular examples include grid search and random search~\citep{bergstra2012random}.
                Since each hyperparameter selection is independent, these algorithms are trivial to parallelize.

            \textbf{Global \ho\!\!:}
                Some \ho algorithms attempt to find a globally optimal hyperparameter setting, which can be important if the loss is non-convex.
                A simple example is random search, while a more sophisticated example is Bayesian optimization~\cite{movckus1975bayesian, snoek2012practical, kandasamy2019tuning}.
                These HO algorithms often involve re-initializing the hyperparameter and weights on each optimization iteration.
                This allows global optimization, at the cost of expensive re-training weights or hyperparameters.

            \textbf{Local \ho\!\!:}
                Other \ho algorithms only attempt to find a locally optimal hyperparameter setting.
                Often these algorithms will maintain a current estimate of the best combination of hyperparameter and weights.
                On each optimization iteration, the hyperparameter is adjusted by a small amount, which allows us to avoid excessive re-training of the weights on each update.
                This is because the new optimal weights are near the old optimal weights due to a small change in the hyperparameters.
            
            \textbf{Learned proxy function based \ho\!\!:}
                Many \ho algorithms attempt to learn a proxy function for optimization.
                The proxy function is used to estimate the loss for a hyperparameter selection.
                We could learn a proxy function for global or local \ho.
                We can learn a useful proxy function over any node in our computational graph including the optimized weights.
                For example, we could learn how the optimized weights change w.r.t. the hyperparameters ~\cite{lorraine2018stochastic}, how the optimized predictions change w.r.t. the hyperparameters~\cite{mackay2019self}, or how the optimized validation loss changes w.r.t. the hyperparameters as in Bayesian Optimization.
                It is possible to do gradient descent on the proxy function to find new hyperparameters to query as in Bayesian optimization.
                Alternatively, we could use a non-differentiable proxy function to get cheap estimates of the validation loss like SMASH~\cite{brock2017smash} for architecture choices.
                
         \begin{table*}[h]\caption{\textbf{Notation}}
            \begin{center}
                \begin{tabular}{c | c}
                    \toprule
                    \HO & Hyperparameter optimization\\
                    \NN & Neural network\\
                    IFT & Implicit Function Theorem\\
                    HVP / JVP & Hessian/Jacobian-vector product\\
                    $\hp, \ep$ & Hyperparameters and \nn parameters/weights\\
                    $\hdim, \edim$ & Hyperparameter and \nn parameter dimensionality\\
                    $\hdom \! \subseteq \! \Real^\hdim,\! \edom \! \subseteq \! \Real^\edim$ & Hyperparameters and \nn parameter domains\\
                    $\hp', \ep'$ & Arbitrary, fixed hyperparameters and weights\\
                    $\Ltr\!(\hp,\! \ep),\!\Lval\!(\hp,\! \ep)\!$& Training loss \& validation loss\\
                    ${\color{blue}\response}$ & Best-response of the weights to the hyperparameters\\
                    ${\color{blue}\widehat{\ep^*}(\hp)}$ & An approximate best-response of the weights to the hyperparameters\\
                    ${\color{red}\Lval^{*}\!(\hp\!)\!} = {\color{red}\Lval\!(\hp,\! \response\!)\!}$& The validation loss with best-responding weights\\
                    {\color{red}Red}& (Approximations to) The validation loss with best-responding weights\\
                    $\edom^* = \ep^*(\hdom)$ & The domain of best-responding weights\\
                    $\hp^*$ & The optimal hyperparameters\\
                    $\vx, \vy$ & An input and its associated target\\
                    $\xdom, \tdom$ & The input and target domains respectively\\
                    $\dataset$ & A data matrix consisting of tuples of inputs and targets\\
                    $\vy(\vx, \ep)$ & A predicted target for a input data and weights\\
                    ${\color{mydarkgreen}\hpderiv{\Lval}}, \epderiv{\Lval}$ & The (validation loss hyperparameter / parameter) direct gradient\\
                    {\color{mydarkgreen}Green} & (Approximations to) The validation loss direct gradient.\\
                    ${\color{blue}\hpderiv{\ep^*}}$ & The best-response Jacobian\\
                    {\color{blue}Blue} & (Approximations to) The (Jacobian of the) best-response of the weights\\
                    & to the hyperparameters\\
                    $\pd{\Lval}{\ep} {\color{blue}\hpderiv{\ep^*}}$\!& The indirect gradient\\
                    $\hpderiv{{\color{red}\Lval^*}}$& A hypergradient: sum of validation losses direct and indirect gradient\!\\
                    ${\color{magenta}\left[ \trainHess{\ep} \right]^{-1}}$ & The training Hessian inverse\\
                    {\color{magenta}Magenta} & (Approximations to) The training Hessian inverse\\
                    ${\color{orange} \epderiv{\Lval} \left[ \trainHess{\ep} \right]^{-1}}$ & The vector - Inverse Hessian product.\\
                    {\color{orange}Orange} & (Approximations to) The vector - Inverse Hessian product.\\
                    $\trainMixed{\ep}$ & The training mixed partial derivatives\\
                    $I$ & The identity matrix\\
                    \bottomrule
                \end{tabular}
            \end{center}
            \label{tab:TableOfNotation}
        \end{table*}
        
        \newcommand{\iftf}{\trainGrad}
        \newcommand{\iftx}{\hparam}
        \newcommand{\ifty}{\ep}
        \newcommand{\iftg}{\ifty^*}
        \newcommand{\ifthdom}{\hdom}
        \newcommand{\iftedom}{\edom}
        \newcommand{\ifta}{\iftx'}
        \newcommand{\iftb}{\ifty'}
        \onecolumn
        \section{Implicit Function Theorem}
            \vspace{-0.2cm}
            \begin{thm*}[Augustin-Louis Cauchy, Implicit Function Theorem]\label{app:IFTComplete}
                \textnormal{
                    Let $\iftf(\iftx, \ifty): \ifthdom \times \iftedom \to \iftedom$ be a continuously differentiable function.
                    Fix a point $(\ifta, \iftb)$ with $\iftf(\ifta, \iftb) = 0$.
                    If the Jacobian $J^{\iftf}_{\ifty}(\ifta, \iftb)$ is invertible, there exists an open set $U \subseteq \ifthdom$ containing $\ifta$ s.t. there exists a continuously differentiable function $\iftg: U \to \iftedom$ s.t.:
                    \begin{equation*}
                        \iftg(\ifta) = \iftb \textnormal{ and } \forall \iftx \in U, \iftf(\iftx, \iftg(\iftx))) = 0
                    \end{equation*}
                    Moreover, the partial derivatives of $\iftg$ in $U$ are given by the matrix product:
                    \begin{equation*}
                        \pd{\iftg}{\iftx} (\iftx) = -\left[ J^{\iftf}_{\ifty}(\iftx, \iftg(\iftx)) \right]^{-1} J^{\iftf}_{\iftx} (\iftx, \iftg(\iftx)))
                    \end{equation*}
                }
            \end{thm*}
            Typically the IFT is presented with $\iftf = f$, $\iftg = g$, $\ifthdom = \Real^{m}$, $\iftedom = \Real^{n}$, $\iftx = x$, $\ifty = y$, $\ifta = a$, $\iftb = b$.
        
        \vspace{-1cm}
        \section{Proofs}\label{app:proofs}
            \begin{lemma*}[1]
                \textnormal{
                    If the recurrence given by unrolling SGD optimization in Eq.~\ref{eqn:sgd_recurrence} has a fixed point $\epA$ (i.e., $0 = \epderiv{\Ltr} |_{\hp, \epA(\hp)}$), then:
                    \begin{equation*}
                        \hpderiv{\epA} = \left. -\left[ \trainHess{\epA(\hp)} \right]^{-1} \trainMixed{\epA(\hp)} \right|_{\epA(\hp)}
                    \end{equation*}
                }
            \end{lemma*}
            \begin{proof}
                \begin{align*}
                    &\Rightarrow \hpderiv{}\left( \left. \epderiv{\Ltr} \right|_{\hp, \epA(\hp)} \right) = 0
                        && \text{given}\\
                    &\Rightarrow \left. \left( \trainMixed{\epA(\hp)} I + \trainHess{\epA(\hp)} \hpderiv{\epA} \right) \right|_{\hp, \epA(\hp)} = 0
                        && \left. \text{chain rule through } \right|_{\hp, \epA(\hp)}\\
                    &\Rightarrow \left. \trainHess{\epA(\hp)} \hpderiv{\epA} \right|_{\hp, \epA(\hp)} = \left. -\trainMixed{\epA(\hp)} \right|_{\hp, \epA(\hp)}
                        && \text{re-arrange terms}\\
                    &\Rightarrow \left. \hpderiv{\epA} \right|_{\hp} = \left. -\left[ \trainHess{\epA(\hp)} \right]^{-1} \trainMixed{\epA(\hp)} \right|_{\hp}
                        && \text{left-multiply by } \left. \left[ \trainHess{\epA(\hp)} \right]^{-1} \right|_{\hp, \epA(\hp)}
                \end{align*}
            \end{proof}
            
            \vspace{1cm}
            \begin{lemma*}[2]
                \textnormal{
                    Given the recurrence from unrolling SGD optimization in Eq.~\ref{eqn:sgd_recurrence} we have:
                    \begin{equation*}
                        \hpderiv{\ep_{i+1}} = -\sum_{j\leq i} \left( \prod_{k<j} I - \left. \trainHess{\ep_{i-k}(\hp)} \right|_{\hp, \ep_{i-k}(\hp)} \right) \left. \trainMixed{\ep_{i-j}(\hp)} \right|_{\hp, \ep_{i-j}(\hp)}
                    \end{equation*}
                }
            \end{lemma*}
            \begin{proof}
                \begin{align*}
                    &\left. \hpderiv{\ep_{i+1}} \right|_{\hp}
                    = \hpderiv{} \left( \left. \ep_{i}(\hp) - \epderiv{\Ltr} \right|_{\hp, \ep_{i}(\hp)} \right)
                        && \text{take derivative w.r.t. } \hp\\
                    &\smallSpace= \left. \hpderiv{\ep_{i}} \right|_{\hp} - \hpderiv{}\left( \left. \epderiv{\Ltr} \right|_{\hp, \ep_{i}(\hp)} \right)
                        && \text{chain rule}\\
                    &\smallSpace= \left. \hpderiv{\ep_{i}} \right|_{\hp} - \left. \left( \trainHess{\ep_{i}(\hp)} \hpderiv{\ep_{i}} + \trainMixed{\ep_{i}(\hp)}\right) \right|_{\hp, \ep_{i}(\hp)}
                        && \left. \text{chain rule through } \right|_{\hp, \ep_{i}(\hp)}\\
                    &\smallSpace= \left. -\trainMixed{\ep_{i}(\hp)} \right|_{\hp, \ep_{i}(\hp)} + \left. \left(I - \trainHess{\ep_{i}(\hp)}\right)\hpderiv{\ep_{i}} \right|_{\hp, \ep_{i}(\hp)}
                        && \text{re-arrange terms}\\
                    &\smallSpace= \left. -\trainMixed{\ep_{i}(\hp)} \right|_{\hp, \ep_{i}(\hp)} + \left. \left(I - \trainHess{\ep_{i}(\hp)}\right) \right|_{\hp, \ep_{i}(\hp)}\cdot\\
                        &\medSpace \left. \left( \!\! \left(\! I - \trainHess{\ep_{i-1}(\hp)} \!\right)\! \hpderiv{\ep_{i-1}} - \trainMixed{\ep_{i-1}(\hp)} \!\right) \right|_{\hp, \ep_{i-1}(\hp)}
                        && \text{expand } \hpderiv{\ep_{i}}\\
                    &\smallSpace= \left. -\trainMixed{\ep_{i}(\hp)} \right|_{\hp, \ep_{i}(\hp)} - \left( \!I - \left. \trainHess{\ep_{i}(\hp)} \right|_{\hp, \ep_{i}(\hp)} \! \right) \! \left. \trainMixed{\ep_{i-1}(\hp)} \right|_{\hp, \ep_{i-1}(\hp)} +\\
                        &\medSpace \left[ \prod_{k<2}I - \left. \trainHess{\ep_{i-k}(\hp)} \right|_{\hp, \ep_{i-k}(\hp)} \right] \left. \hpderiv{\ep_{i-1}} \right|_{\hp}
                        && \text{re-arrange terms}\\
                    &\smallSpace= \cdots 
                        && \text{}\\
                    &\textnormal{So, }\hpderiv{\ep_{i+1}} = -\sum_{j\leq i} \left[ \prod_{k<j} I - \left. \trainHess{\ep_{i-k}(\hp)} \right|_{\hp, \ep_{i-k}(\hp)} \right] \left. \trainMixed{\ep_{i-j}(\hp)} \right|_{\hp, \ep_{i-j}(\hp)}
                        && \text{telescope the recurrence}
                \end{align*}
            \end{proof}
            \begin{thm*}[Neumann-SGD]
                \textnormal{
                    Given the recurrence from unrolling SGD optimization in Eq.~\ref{eqn:sgd_recurrence}, if $\ep_0 = \ep^*(\hp)$:
                    \begin{equation*}
                        \hpderiv{\ep_{i+1}} = -\left. \left( \sum_{j \leq i} \left[I - \trainHess{\ep(\hp)}\right]^j \right) \trainMixed{\ep(\hp)} \right|_{\ep^*(\hp)}
                    \end{equation*}
                }
                and if \textnormal{$I + \trainHess{\ep(\hp)}$} is contractive:
                \textnormal{
                    \begin{equation*}
                        \lim_{i\to\infty} \hpderiv{\ep_{i+1}} = -\left. \left[ \trainHess{\ep(\hp)} \right]^{-1} \trainMixed{\ep(\hp)} \right|_{\ep^*(\hp)}
                    \end{equation*}
                }
            \end{thm*}
            \begin{proof}
                \begin{align*}
                    &\lim_{i\to\infty} \left. \hpderiv{\ep_{i+1}} \right|_{\hp}
                        && \text{take} \lim_{i\to\infty} \\
                    &\smallSpace= \lim_{i\to\infty} \left( -\sum_{j\leq i} \left[ \prod_{k<j} I - \left. \trainHess{\ep_{i-k}(\hp)} \right|_{\hp, \ep_{i-k}(\hp)} \right] \left. \trainMixed{\ep_{i-j}(\hp)} \right|_{\hp, \ep_{i-j}(\hp)} \right)
                        && \text{by Lemma 2}\\
                    &\smallSpace= -\lim_{i\to\infty} \left. \left( \sum_{j\leq i} \left[ \prod_{k<j} I - \trainHess{\ep(\hp)} \right] \trainMixed{\ep(\hp)} \right) \right|_{\hp, \ep^*(\hp)}
                        && \ep_0 = \ep^*(\hparam) = \ep_i \\
                    &\smallSpace= -\lim_{i\to\infty} \left. \left( \sum_{j \leq i} \left[I - \trainHess{\ep(\hp)}\right]^j \right) \trainMixed{\ep(\hp)} \right|_{\hp, \ep^*(\hp)}
                        && \text{simplify} \\
                    &\smallSpace= -\left. \left[I - \left( I - \trainHess{\ep(\hp)} \right) \right]^{-1} \trainMixed{\ep(\hp)} \right|_{\hp, \ep^*(\hp)}
                        && \text{contractive \& Neumann series} \\
                    &\smallSpace= -\left. \left[ \trainHess{\ep(\hp)} \right]^{-1} \trainMixed{\ep(\hp)} \right|_{\hp, \ep^*(\hp)}
                        && \text{simplify}
                \end{align*}
            \end{proof}
        
        \twocolumn
            
        \section{Experiments}\label{app:experiments}
            We use PyTorch~\cite{paszke2017automatic} as our computational framework.
            All experiments were performed on NVIDIA TITAN Xp GPUs.
        
            For all CNN experiments we use the following optimization setup:
            for the \nn weights we use Adam~\cite{kingma2014adam} with a learning rate of 1e-4.
            For the hyperparameters we use RMSprop~\cite{hinton2012neural} with a learning rate of 1e-2.
            \subsection{Overfitting a Small Validation Set}
                We see our algorithm's ability to overfit the validation data (see Fig.~\ref{fig:overfit_valid_all}).
                We use $\numOverfitTrain$ training input, and $\numOverfitValid$ validation input with the standard testing partition for both MNIST and CIFAR-\num{10}.
                We check performance with logistic regression (Linear), a 1-Layer fully-connected \nn with as many hidden units as input size (ex., $\num{28} \times \num{28} = \num{784}$, or $\num{32} \times \num{32} \times \num{3} = \num{3072}$), LeNet~\cite{lecun1998gradient}, AlexNet~\cite{krizhevsky2012imagenet}, and ResNet\num{44}~\cite{he2016deep}.
                In all examples we can achieve \SI{100}{\percent} training and validation accuracy, while the testing accuracy is significantly lower.
                \begin{figure}[h!]
                    \centering
                    \begin{tikzpicture}
                        \centering
                        \node (img){\includegraphics[trim={0.15cm 0cm 0.15cm 0cm},clip, width=.48\linewidth]{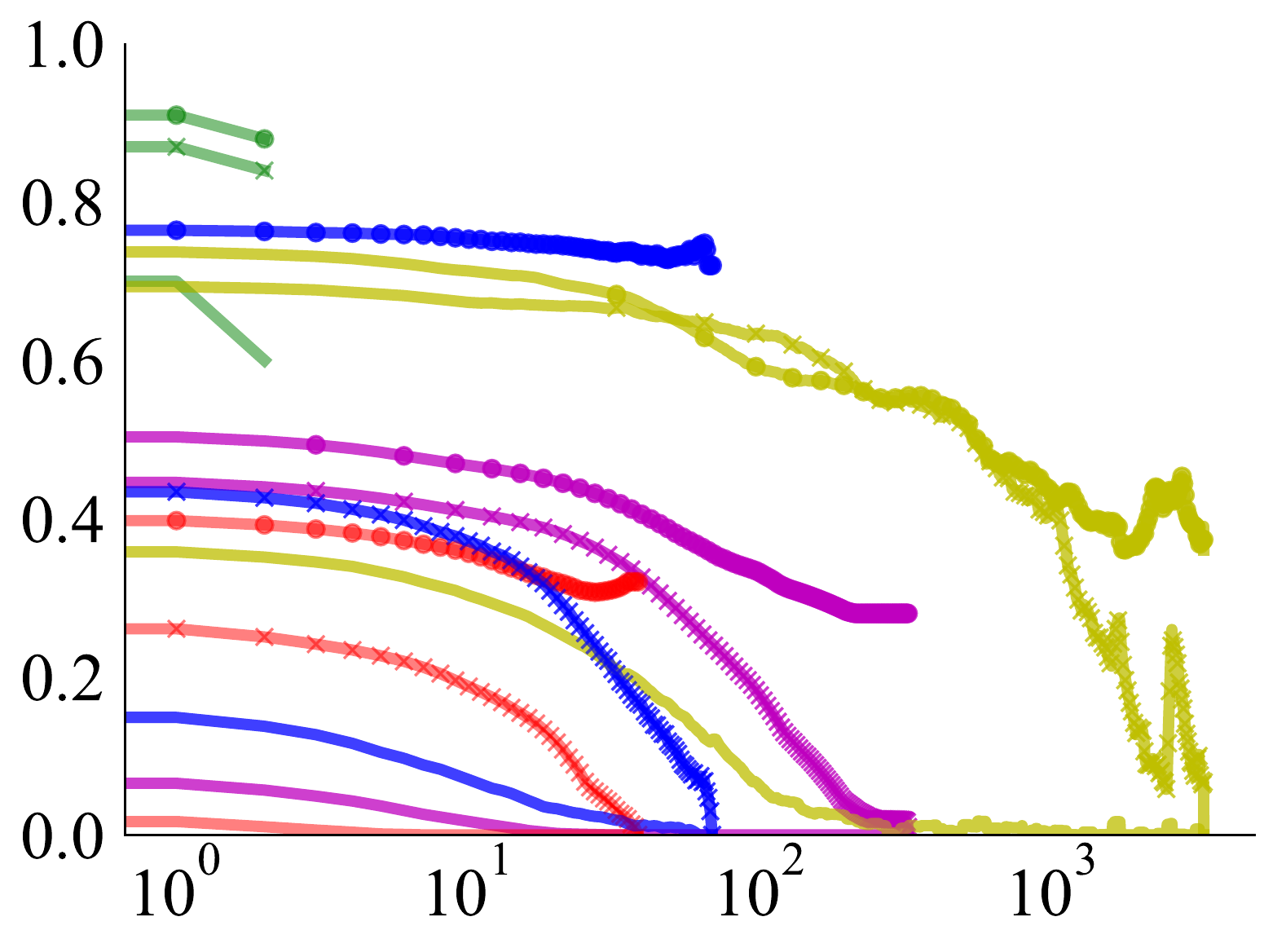}};
                        \node[left=of img, node distance=0cm, rotate=90, xshift=1.75cm, yshift=-.90cm, font=\color{black}] {Classification Error};
                        \node[above=of img, node distance=0cm, yshift=-1.25cm,font=\color{black}] {MNIST};
                        \node[below=of img, node distance=0cm, yshift=1.25cm,font=\color{black}] {Iteration};
                        \node (img2)[right=of img, xshift=-1.3cm]{\includegraphics[trim={1.3cm 0cm 0.15cm 0cm},clip, width=.455\linewidth]{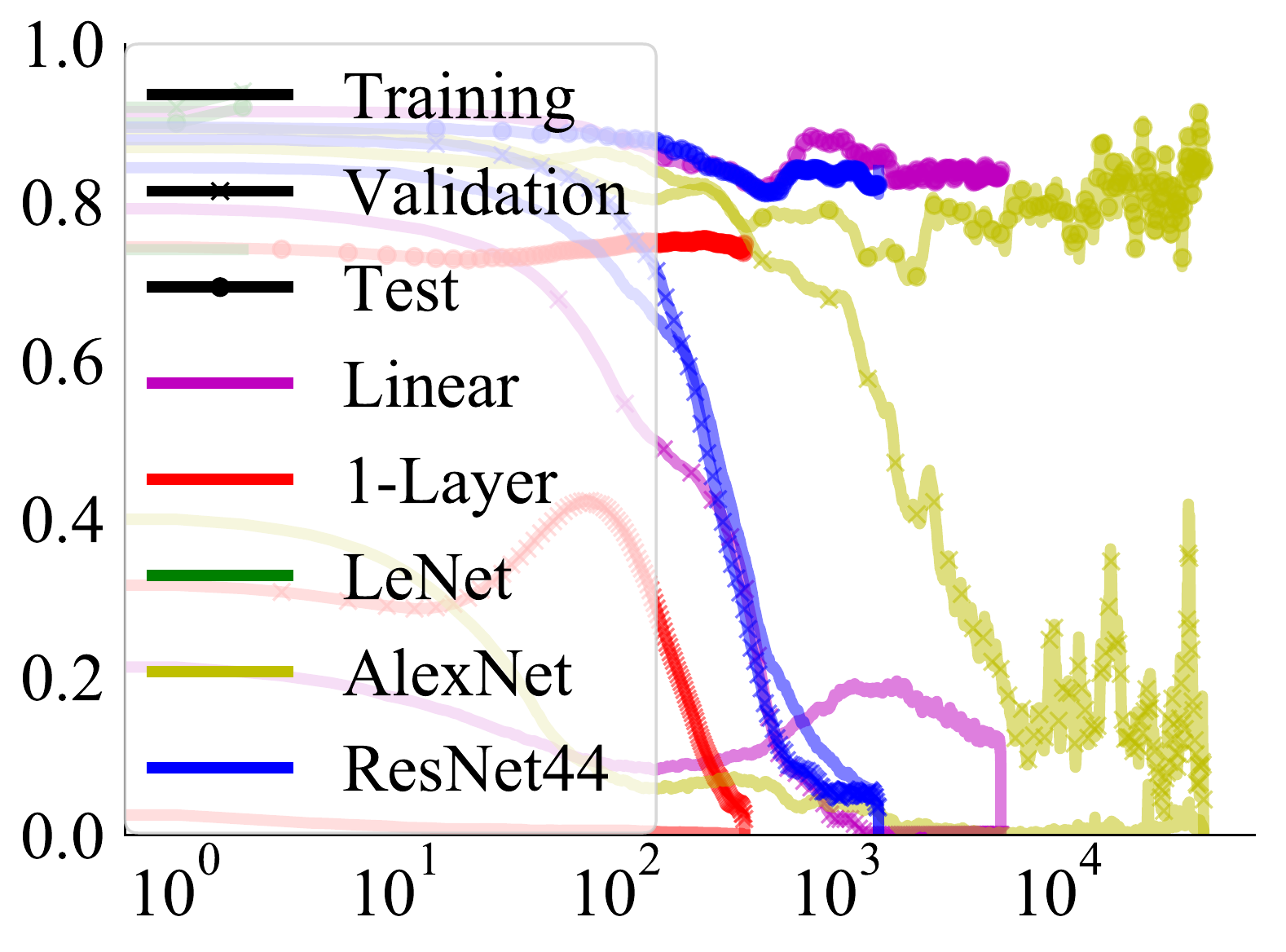}};
                        \node[above=of img2, node distance=0cm, yshift=-1.25cm,font=\color{black}] {CIFAR-\num{10}};
                        \node[below=of img2, node distance=0cm, yshift=1.25cm,font=\color{black}] {Iteration};
                    \end{tikzpicture}
                    \vspace{-0.6cm}
                    \caption{{Overfitting validation data.}
                        {Algorithm~\ref{alg:joint_train} can overfit the validation dataset.}
                        We use $\numOverfitTrain$ training input, and $\numOverfitValid$ validation input with the standard testing partition for both MNIST and CIFAR-\num{10}.
                        We check the performance with logistic regression (Linear), a 1-Layer fully-connected \nn with as many hidden units as input size (ex., $\num{28} \times \num{28} = \num{784}$, or $\num{32} \times \num{32} \times \num{3} = \num{3072}$), LeNet~\cite{lecun1998gradient}, AlexNet~\cite{krizhevsky2012imagenet}, and ResNet\num{44}~\cite{he2016deep}.
                        Separate lines are plotted for the training, validation, and testing error.
                        In all examples we achieve \SI{100}{\percent} training and validation accuracy, while the testing accuracy is significantly lower.
                    }
                    \label{fig:overfit_valid_all}
                \end{figure}
            
            \newcommand{\numInverseTrain}{\num{50}}
            \newcommand{\numInverseValid}{\num{50}}
            \newcommand{\inverseBatch}{\num{50}}
            
            \subsection{Dataset Distillation}
                With MNIST we use the entire dataset in validation, while for CIFAR we use $\numDistillValid$ validation data points.
                \begin{figure*}
                    \centering
                    \begin{tikzpicture}
                        \centering
                        \node (img){\includegraphics[trim={3cm 2.5cm 2.5cm 3cm},clip, width=.98\linewidth]{learned_images_CIFAR100_100_mlp.pdf}};
                        \node[above=of img, node distance=0cm, yshift=-1.cm,font=\color{black}] {\textbf{CIFAR-100 Distillation}};
                    \end{tikzpicture}
                    \vspace{-0.02\textheight}
                    \caption{
                        The complete dataset distillation for CIFAR-100.
                        Referenced in Fig.~\ref{fig:distilled_cifar}.
                    }
                    \label{app:cifar-100-distillation}
                \end{figure*}
                
            \subsection{Learned Data Augmentation}
                \textbf{Augmentation Network Details:}
                    Data augmentation can be framed as an image-to-image transformation problem; inspired by this, we use a U-Net~\cite{ronneberger2015u} as the data augmentation network.
                    To allow for stochastic transformations, we feed in random noise by concatenating a noise channel to the input image, so the resulting input has 4 channels.

            \subsection{RNN Hyperparameter Optimization}
                \label{app:exp_rnn}
                Our base our implementation on the AWD-LSTM codebase \url{https://github.com/salesforce/awd-lstm-lm}.
                Similar to~\cite{gal2016theoretically} we used a \num{2}-layer LSTM with \num{650} hidden units per layer and \num{650}-dimensional word embeddings.

                \paragraph{Overfitting Validation Data:}
                We used a subset of \num{10} training sequences and 10 validation sequences, and tuned separate weight decays per parameter.
                The LSTM architecture we use has \num{13280400} weights, and thus an equal number of weight decay hyperparameters.

                \paragraph{Optimization Details:}
                For the large-scale experiments, we follow the training setup proposed in \cite{merity2017regularizing}: for the \nn weights, we use SGD with learning rate \num{30} and gradient clipping to magnitude \num{0.25}.
                The learning rate was decayed by a factor of 4 based on the nonmonotonic criterion introduced by~\cite{merity2017regularizing} (i.e., when the validation loss fails to decrease for 5 epochs).
                To optimize the hyperparameters, we used Adam with learning rate \num{0.001}.
                We trained on sequences of length 70 in mini-batches of size 40.
\end{document}

%% file: math_commands.tex
\usepackage{amsmath,amsfonts,bm}

\def\eqref#1{equation~\ref{#1}}

\def\1{\bm{1}}

\def\vx{{\bm{x}}}
\def\vy{{\bm{y}}}

\DeclareMathAlphabet{\mathsfit}{\encodingdefault}{\sfdefault}{m}{sl}
\SetMathAlphabet{\mathsfit}{bold}{\encodingdefault}{\sfdefault}{bx}{n}

\DeclareMathOperator*{\argmin}{arg\,min}